%% file: main.tex
\documentclass[11pt]{article} 

\usepackage{microtype}
\usepackage{graphicx}
\usepackage{subfigure}
\usepackage{booktabs}

\usepackage{hyperref}

\usepackage{ifthen}
\newboolean{arxivVersion}
\setboolean{arxivVersion}{true} 

\ifthenelse{\boolean{arxivVersion}}{
\usepackage[margin=1.0in]{geometry}
\usepackage{xcolor}
\title{Training Quantized Neural Networks to Global Optimality via Semidefinite Programming}
\author{%
   \textbf{Burak Bartan} \\
  \texttt{bbartan@stanford.edu} \\
   \and
  \textbf{Mert Pilanci} \\
  \texttt{pilanci@stanford.edu} \\
  }
\date{Department of Electrical Engineering, 
  Stanford University \\ \today}
}
{
\usepackage{icml2021}
\icmltitlerunning{Training Quantized Neural Networks via Semidefinite Programming}
}

\input{latex_defns.tex}

\begin{document}

\ifthenelse{\boolean{arxivVersion}}{
\maketitle
}
{
\twocolumn[
\icmltitle{Training Quantized Neural Networks to Global Optimality via Semidefinite Programming}


\icmlsetsymbol{equal}{*}

\begin{icmlauthorlist}
\icmlauthor{Aeiau Zzzz}{equal,to}
\icmlauthor{Bauiu C.~Yyyy}{equal,to,goo}
\icmlauthor{Cieua Vvvvv}{goo}
\icmlauthor{Iaesut Saoeu}{ed}
\icmlauthor{Fiuea Rrrr}{to}
\icmlauthor{Tateu H.~Yasehe}{ed,to,goo}
\icmlauthor{Aaoeu Iasoh}{goo}
\icmlauthor{Buiui Eueu}{ed}
\icmlauthor{Aeuia Zzzz}{ed}
\icmlauthor{Bieea C.~Yyyy}{to,goo}
\icmlauthor{Teoau Xxxx}{ed}
\icmlauthor{Eee Pppp}{ed}
\end{icmlauthorlist}

\icmlaffiliation{to}{Department of Computation, University of Torontoland, Torontoland, Canada}
\icmlaffiliation{goo}{Googol ShallowMind, New London, Michigan, USA}
\icmlaffiliation{ed}{School of Computation, University of Edenborrow, Edenborrow, United Kingdom}

\icmlcorrespondingauthor{Cieua Vvvvv}{c.vvvvv@googol.com}
\icmlcorrespondingauthor{Eee Pppp}{ep@eden.co.uk}

\icmlkeywords{Machine Learning, ICML}
\vskip 0.3in
]


\printAffiliationsAndNotice{\icmlEqualContribution} 
}

\input{abstract.tex}
\input{introduction.tex}
\input{lifting.tex}
\input{duality.tex}
\input{sampling_algorithm.tex}
\input{numerical_results.tex}
\input{conclusion.tex}

\bibliography{refs,refsRTML}

\ifthenelse{\boolean{arxivVersion}}
{
\bibliographystyle{plain}
}{
\bibliographystyle{icml2021}
}

\input{supp_file.tex}

\end{document}

%% file: latex_defns.tex
\usepackage{amsmath}
\usepackage{mathtools}
\usepackage{amsfonts}
\usepackage{tensor}
\usepackage{enumitem}
\usepackage{amsthm}

\DeclareMathOperator{\tr}{tr}
\DeclareMathOperator{\rank}{rank}

\DeclareMathOperator{\ones}{\bar{1}}
\DeclareMathOperator{\diag}{diag}
\DeclareMathOperator{\sign}{sign}
\newtheorem{theorem}{Theorem}
\newtheorem{remark}{Remark}

\usepackage[ruled, noline]{algorithm2e}

\DeclareMathOperator{\Exs}{\mathbb{E}}

%% file: abstract.tex
\begin{abstract}
Neural networks (NNs) have been extremely successful across many tasks in machine learning. Quantization of NN weights has become an important topic due to its impact on their energy efficiency, inference time and deployment on hardware. Although post-training quantization is well-studied, training optimal quantized NNs involves combinatorial non-convex optimization problems which appear intractable. In this work, we introduce a convex optimization strategy to train quantized NNs with polynomial activations. Our method leverages hidden convexity in two-layer neural networks from the recent literature, semidefinite lifting, and Grothendieck's identity. Surprisingly, we show that certain quantized NN problems can be solved to global optimality in polynomial-time in all relevant parameters via semidefinite relaxations. We present numerical examples to illustrate the effectiveness of our method.
\end{abstract}

%% file: introduction.tex
\section{Introduction} \label{sec:introduction}

In this paper we focus on training quantized neural networks for efficient machine learning models. We consider the combinatorial and non-convex optimization of minimizing empirical error with respect to quantized weights. We focus on polynomial activation functions, where the training problem is still non-trivial to solve.

Recent work has shown that two-layer neural networks with ReLU \cite{pilanci2020neural,sahiner2020vector} and leaky ReLU activations \cite{lacotte2020all} can be trained via convex optimization in polynomial time with respect to the number of samples and neurons. Moreover, degree-two polynomial activations can be trained to global optimality in polynomial time with respect to all problem dimensions using semidefinite programming \cite{bartan2021neural}. In this work, we take a similar convex duality approach that involves semidefinite programming. However, our method and theoretical analysis are substantially different since we consider quantized weights, which involves a discrete non-convex optimization problem. The fact that the first layer weights are discrete renders it a combinatorial NP-hard problem and thus we cannot hope to obtain a similar result as in \cite{bartan2021neural} or \cite{pilanci2020neural}. In contrast, in \cite{bartan2021neural} it is shown that with the constraint $\|u_j\|_2=1$ and $\ell_1$-norm regularization on the second layer weights, the global optimum can be found in fully polynomial time and that the problem becomes NP-hard in the case of quadratic regularization (i.e. weight decay).

The approach that we present in this paper for training quantized neural networks is different from other works in the quantized neural networks literature. In particular, our approach involves deriving a semidefinite program (SDP) and designing a sampling algorithm based on the solution of the SDP. Our techniques lead to a provable guarantee for the difference between the resulting loss and the optimal non-convex loss for the first time. 

\subsection{Prior work}

Recently, there has been a lot of research effort in the realm of compression and quantization of neural networks for real-time implementations. In \cite{ternary_quantization}, the authors proposed a method that reduces network weights into ternary values by performing training with ternary values. Experiments in \cite{ternary_quantization} show that their method does not suffer from performance degradation and achieve 16x compression compared to the original model. 
The authors in \cite{compress_quantization} focus on compressing dense layers using quantization to tackle model storage problems for large-scale models. The work presented in \cite{deep_compression} also aims to compress deep networks using a combination of pruning, quantization and Huffman coding. In \cite{fixed_point_quantization}, the authors present a quantization scheme where they use different bit-widths for different layers (i.e., bit-width optimization). Other works that deal with fixed point training include \cite{challenges_fixed_point}, \cite{GuptaAGN15}, \cite{fixedpoint_hwang}. Furthermore, \cite{anwar2015} proposes layer-wise quantization based on $\ell_2$-norm error minimization followed by retraining of the quantized weights. However, these studies do not address optimal approximation. In comparison, our approach provides optimal quantized neural networks.

In \cite{allenzhu2020backward}, it was shown that the degree two polynomial activation functions perform comparably to ReLU activation in practical deep networks. Specifically, it was reported in \cite{allenzhu2020backward} that for deep neural networks, ReLU activation achieves a classification accuracy of $0.96$ and a degree two polynomial activation yields an accuracy of $0.95$ on the Cifar-10 dataset. Similarly for the Cifar-100 dataset, it is possible to obtain an accuracy of $0.81$ for ReLU activation and $0.76$ for the degree two polynomial activation. These numerical results are obtained for the activation $\sigma(t)=t+0.1t^2$. Furthermore, in encrypted computing, it is desirable to have low degree polynomials as activation functions. For instance, homomorphic encryption methods can only support additions and multiplications in a straightforward way. These constraints make low degree polynomial activations attractive for encrypted networks. In \cite{gilad2016cryptonets}, degree two polynomial approximations were shown to be effective for accurate neural network predictions with encryption. These results demonstrate that polynomial activation neural networks are a promising direction for further exploration.

Convexity of infinitely wide neural networks was first considered in \cite{bengio2006convex} and later in \cite{bach2017breaking}. A convex geometric characterization of finite width neural networks was developed in \cite{ergen2020aistats,ergen2019shallow,bartan2019convex}. Exact convex optimization representations of finite width two-layer ReLU neural network problems were developed first in \cite{pilanci2020neural} and extended to leaky ReLU \cite{lacotte2020all} and polynomial activation functions \cite{bartan2021neural}. These techniques were also extended to other network architectures including three-layer ReLU \cite{ergen2020implicit}, autoencoder \cite{sahiner2020convex}, autoregressive \cite{guptaexact}, batch normalization \cite{ergen2021demystifying} and deeper networks \cite{ergen2020convexdeep}.

\subsection{Notation}
We use $X \in \mathbb{R}^{n \times d}$ to denote the data matrix throughout the text, where its rows $x_i \in \mathbb{R}^d$ correspond to data samples and columns are the features. $y \in \mathbb{R}^{n}$ denotes the target vector. We use $\ell(\hat{y}, y)$ for convex loss functions where $\hat{y}$ is the vector of predictions and $\ell^*(v) = \sup_{z}(v^Tz - \ell(z, y))$ denotes its Fenchel conjugate. 
$\tr$ refers to matrix trace. $\sign(\cdot)$ is the elementwise sign function.
We use the notation $Z \succeq 0$ for positive semidefinite matrices (PSD). 
We use $\circ$ for the Hadamard product of vectors and matrices. The symbol $\otimes$ denotes the Kronecker product. We use $\lambda_{\max}(\cdot)$ to denote the largest eigenvalue of its matrix argument. If the input to $\diag(\cdot)$ is a vector, then the result is a diagonal matrix with its diagonal entries equal to the entries of the input vector. If the input to $\diag(\cdot)$ is a matrix, then the result is a vector with entries equal to the diagonal entries of the input matrix. $\ones$ refers to the vector of $1$'s.
$\mathbb{S}^{d\times d}$ represents the set of $(d\times d)$-dimensional symmetric matrices.

%% file: lifting.tex
\section{Lifting Neural Network Parameters} \label{sec:lifting}
We focus on two-layer neural networks with degree two polynomial activations $\sigma(t) := at^2 + bt + c$. Let $f: \mathbb{R}^d \rightarrow \mathbb{R}$ denote the neural network defined as
\begin{align}
    f(x) = \sum_{j=1}^m \sigma(x^Tu_j) \alpha_j
    \label{eq:polyactnetwork}
\end{align}
where $x\in \mathbb{R}^d$ is the input sample, $u_j \in \mathbb{R}^d$ and $\alpha_j \in \mathbb{R}$ are the first and second layer weights, respectively. This is a fully connected neural network with $m$ neurons in the hidden layer.
We focus on the setting where the first $dm$ weights (i.e., $u_j \in \mathbb{R}^d$, $j=1,\dots,m$) in the hidden layer are constrained to be integers.

The results are extended to neural networks with vector outputs, i.e., $f: \mathbb{R}^d \rightarrow \mathbb{R}^C$, in Section \ref{sec:vector_output} of the Appendix.

\subsection{Bilinear activation networks}
Now we introduce a simpler architecture with bilinear activations $\mathcal{X} \rightarrow u^T \mathcal{X} v$ and binary quantization given by 
\begin{align}
    f^\prime(\mathcal{X}) = \sum_{j=1}^{m^\prime} u_j^T \mathcal{X} v_j \alpha_j \quad \mbox{with}  \quad u_j,v_j \in \{-1,+1\}^d,\,\, \alpha_j \in \mathbb{R}, \,\,\forall j
\end{align}
where $\mathcal{X} \in\mathbb{R}^{d\times d}$ is the lifted version of the input $x \in \mathbb{R}^d$ as will be defined in the sequel. We show that this architecture is sufficient to represent multi-level integer quantization and degree two polynomial activations without any loss of generality. In addition, these networks can be mapped to the standard network in \eqref{eq:polyactnetwork} in a straightforward way as we formalize in this section. Hence, some of our results leverage the above architecture for training and transform a bilinear activation network into a polynomial activation network.

\begin{theorem}[Reduction to binary quantization and bilinear activation] \label{thm:reduction_to_binary}
The following multi-level (i.e. $M+1$ levels) quantized neural network
\begin{align*}
   &f(x) = \sum_{j=1}^m \sigma(x^Tu_j) \alpha_j \\ &\mbox{ where } u_j\in\{-M,-M+2,\dots,0,\dots,M-2,M\}^d,\alpha_j\in\mathbb{R},\,\forall j
\end{align*}
can be represented as a binary quantized bilinear activation network
\begin{align*}
    f^\prime(\mathcal{X}) = \sum_{j=1}^{m^\prime} u_j^T \mathcal{X} v_j \alpha_j \mbox{ where } u_j,v_j\in\{-1,+1\}^{dM+1},
\end{align*}
$\mathcal{X} := \left[\begin{array}{cc} a\tilde x\tilde x^T  & \frac{b}{2}\tilde x\\ \frac{b}{2}\tilde x^T & c  \end{array}\right]$ and $\tilde x:=x\otimes 1_{M}$. Conversely, any binary quantized bilinear activation network $f^\prime(\mathcal{X})$ of this form can be represented as a multi-level quantized neural network $f(x)$.
\end{theorem}
%

In the remainder of this section, we provide a constructive proof of the above theorem by showing the reduction in three steps: Reducing to binary quantization, lifting and reducing to bilinear activation.

\subsection{Reducing multi-level quantization to binary} \label{subsec:multilevel}
In this section, we show that the two level binary quantization $\{-1,1\}$ model is sufficient to model other quantization schemes with integer levels. Hence, we can focus on binary quantized neural network models without loss of generality.
Suppose that $q$ represents a hidden neuron quantized to $M+1$ levels given by
\begin{align}
    q\in\mathcal{Q}_{M}^d:=\{-M, -M+2,\dots,0,\dots,M-2,M\}^d\,.
\end{align}
Then we equivalently have
\begin{align}
    q^T x = \sum_{i=1}^d q_i x_i = \sum_{i=1}^d \sum_{k=1}^M u_{k+(i-1)M} x_i = u^T \tilde x\,, 
\end{align}
where $\tilde x = x\otimes 1_{M} = [x_1,x_1,\dots,x_2,x_2,\dots,]^T \in \mathbb{R}^{dM}$ since $\sum_{k=1}^M u_{k+(i-1)M}\in \mathcal{Q}_{M}\,\forall i$. Therefore, stacking the input data $x$ by replication as $\tilde x \in \mathbb{R}^{dM}$ enables $M+1$ level quantization to be represented as binary quantization.
\subsection{Lifting dimensions}
We first show that binary quantized networks with degree two polynomial activations are equivalent to binary quantized networks with quadratic activations. Note that the network output can be expressed as
\begin{align}
    f(x) &= \sum_{j=1}^m \big(a(x^Tu_j)^2 + b (x^Tu_j) + c\big)\alpha_j \nonumber \\
         &= \sum_{j=1}^m \tilde u^T_j \left[\begin{array}{cc} axx^T  & \frac{b}{2}x\\ \frac{b}{2}x^T & c  \end{array}\right] \tilde u_j\alpha_j
\end{align}
where we defined the augmented weight vectors $\tilde u_j:=[u^T_j, 1]^T$. Consequently, we can safely relax this constraint to $\tilde u_j\in\{-1,+1\}^{d+1}$ since $\tilde u^T_j \left[\begin{array}{cc} axx^T  & \frac{b}{2}x\\ \frac{b}{2}x^T & c  \end{array}\right] \tilde u_j = (-\tilde u_j)^T \left[\begin{array}{cc} axx^T  & \frac{b}{2}x\\ \frac{b}{2}x^T & c  \end{array}\right] (-\tilde u_j)$ and we can assume $(\tilde u_j)_{d+1} = 1$ without loss of generality.
\subsection{Reduction to bilinear activation}
Now we show that we can consider the network model 
\begin{align}
    f(x) &= \sum_{j=1}^m  u^T_j \underbrace{\left[\begin{array}{cc} axx^T  & \frac{b}{2}x\\ \frac{b}{2}x^T & c  \end{array}\right]}_{\mathcal{X}}  v_j\alpha_j = \sum_{j=1}^m u^T_j \mathcal{X}  v_j\alpha_j
\end{align}
where $\{u_j,v_j\}_{j=1}^m$ are the model parameters to represent networks with quadratic activations without loss of generality.
Using the symmetrization identity
\begin{align} \label{eq:symmetrization_identity}
    2u^T A v = (u+v)^T A (u+v) - u^T A u - v^T A v \,,
\end{align}
we can express the neural network output as
\begin{align*}
    2f(x) &= \sum_{j=1}^m \left(  (u_j+v_j)^T \left[\begin{array}{cc} axx^T  & \frac{b}{2}x\\ \frac{b}{2}x^T & c  \end{array}\right]  (u_j+v_j)\alpha_j -  u_j^T \left[\begin{array}{cc} axx^T  & \frac{b}{2}x\\ \frac{b}{2}x^T & c  \end{array}\right]  u_j\alpha_j -  v_j^T \left[\begin{array}{cc} axx^T  & \frac{b}{2}x\\ \frac{b}{2}x^T & c  \end{array}\right]  v_j\alpha_j \right)\,.
\end{align*}
Note that $\frac{1}{2} (u_j+v_j) \in \{-1,0,1\}^d$ and the above  can be written as a quantized network with quadratic activation and $3m$ hidden neurons.

%% file: duality.tex
\section{Convex Duality of Quantized Neural Networks and SDP Relaxations} \label{sec:duality}
We consider the following non-convex training problem for the two-layer polynomial activation network
\begin{align}
    p^* = &\min_{\mbox{s.t.}\, u_j\in\{-1,1\}^d,\alpha_j\in\mathbb{R}\,j\in[m]}  \ell \left(\sum_{j=1}^m \sigma(Xu_j) \alpha_j, \, y \right) + \beta d \sum_{j=1}^m \vert\alpha_j\vert \,.
    \label{eq:nonconvexpolyact}
\end{align}
Here, $\ell(\cdot,y)$ is a convex and Lipschitz loss function, $\sigma(t) := at^2 + bt + c$ is a degree-two polynomial activation function, $m$ is the number of neurons, $\beta$ is the regularization parameter. 

It is straightforward to show that this optimization problem is NP-hard even for the case when $m=1$, $\sigma(t) = t$ is the linear activation and $\ell(u,y)=(u-y)^2$ is the squared loss via a reduction to the MaxCut problem \cite{gwmaxcut}.

We scale the regularization term by $d$ to account for the fact that the hidden neurons have Euclidean norm $\sqrt{d}$, which simplifies the notation in the sequel. Taking the convex dual with respect to the second layer weights $\{\alpha_j\}_{j=1}^m$, the optimal value of the primal is lower bounded by
\begin{align} \label{eq:max_problem_initial}
    p^* \geq d^* &=\max_{|v^T \sigma(Xu)| \leq \beta d \,, \forall u\in\{-1,1\}^d} -\ell^*(-v) \nonumber \\
    &= \max_{\max_{u: u\in\{-1,1\}^d} |v^T \sigma(Xu)| \leq \beta d} -\ell^*(-v).
\end{align}
Remarkably, the above dual problem is a convex program since the constraint set is an intersection of linear constraints. However, the number of linear constraints is exponential due to the binary quantization constraint. 

We now describe an SDP relaxation which provides a lower-bounding and tractable dual convex program. Our formulation is inspired by the SDP relaxation of MaxCut \cite{gwmaxcut}, which is analogous to the constraint subproblem in \eqref{eq:max_problem_initial}. Let us assume that the activation is quadratic $\sigma(u)=u^2$, since we can reduce degree two polynomial activations to quadratics without loss of generality as shown in the previous section. Then, we have $|v^T \sigma(Xu)| = |u^T (\sum_{i=1}^n v_ix_ix_i^T) u|$.

The constraint $\max_{u:u_i^2=1,\forall i}|v^T(Xu)^2| \leq \beta d$ can be equivalently stated as the following two inequalities
\begin{align}
    q_1^* &= \max_{u:u_i^2=1,\forall i} u^T \left(\sum_{i=1}^n v_ix_ix_i^T\right) u \leq \beta d \,,  \nonumber \\
    q_2^* &= \max_{u:u_i^2=1,\forall i} u^T \left(-\sum_{i=1}^n v_ix_ix_i^T\right) u \leq \beta d \,.
\end{align}
The SDP relaxation for the maximization $\max_{u:u_i^2=1,\forall i} u^T \left(\sum_{i=1}^n v_ix_ix_i^T\right) u$ is given by
\begin{align} \label{eq:sdp_relaxation}
    \hat{q}_1 = \max_{U \succeq 0,\, U_{ii}=1,\forall i} \tr\left( \sum_{i=1}^n v_i x_ix_i^T U \right) \,,
\end{align}
where $U \in \mathbb{S}^{d \times d}$. This is a relaxation since we removed the constraint $\rank(U)=1$. Hence, the optimal value of the relaxation is an upper bound on the optimal solution, i.e., $\hat{q}_1 \geq q_1^*$. Consequently, the relaxation leads to the following lower bound:
\begin{align}
    d^* &\geq \max_{q_1^* \leq \beta d, \, q_2^* \leq \beta d} -\ell^*(-v) \geq \max_{\hat{q}_1 \leq \beta d, \, \hat{q}_2 \leq \beta d} -\ell^*(-v) \,.
\end{align}
More precisely, we arrive at $d^*\ge d_{\mathrm{SDP}}$ where
\begin{align} \label{eq:dual_w_max_constraints}
    d_{\mathrm{SDP}}:=\max_{v} & \, -\ell^*(-v) \nonumber \\
    \mbox{s.t.} \, &\max_{U \succeq 0,\, U_{ii}=1,\forall i} \tr\left( \sum_{i=1}^n v_i x_ix_i^T U \right) \leq \beta d \nonumber \\
    &\max_{U \succeq 0,\, U_{ii}=1,\forall i} \tr\left(- \sum_{i=1}^n v_i x_ix_i^T U \right) \leq \beta d \,.
\end{align}

The dual of the SDP in the constraint \eqref{eq:sdp_relaxation} is given by the dual of the MaxCut SDP relaxation, which can be stated as
\begin{align} \label{eq:dual}
    \min_{z \in \mathbb{R}^d} \quad & d \cdot \lambda_{\max}\left(\sum_{i=1}^n v_i x_ix_i^T + \diag(z)\right) \nonumber \\
    \mbox{s.t.} \quad & \ones^T z = 0 \,.
\end{align}
Since the primal problem is strictly feasible, it follows from Slater's condition that the strong duality holds between the primal SDP and the dual SDP. This allows us to reformulate the problem in \eqref{eq:dual_w_max_constraints} as
%
\begin{align}
    \max_{v,z_1,z_2} \quad &-\ell^*(-v) \nonumber \\
    \mbox{s.t.} \quad & \lambda_{\max}\left(\sum_{i=1}^n v_i x_ix_i^T + \diag(z_1)\right) \leq \beta \nonumber \\
    & \lambda_{\max}\left(-\sum_{i=1}^n v_i x_ix_i^T + \diag(z_2)\right) \leq \beta \nonumber \\
    & \ones^T z_1 = 0 ,\,\, \ones^T z_2 = 0 \,,
\end{align}
where the variables have dimensions $v\in\mathbb{R}^n$, $z_1, z_2 \in\mathbb{R}^d$ and $\lambda_{\max}$ denotes the largest eigenvalue. Expressing the largest eigenvalue constraints as linear matrix inequalities yields
\begin{align}
    d_{\mathrm{SDP}}:=\max_{v,z_1,z_2} \quad &-\ell^*(-v) \nonumber \\
    \mbox{s.t.} \quad & \sum_{i=1}^n v_i x_ix_i^T + \diag(z_1) - \beta I_d \preceq 0 \nonumber \\
    & -\sum_{i=1}^n v_i x_ix_i^T + \diag(z_2) - \beta I_d \preceq 0 \nonumber \\
    & \ones^T z_1 = 0 ,\,\, \ones^T z_2 = 0 \,.
\end{align}
Next, we will find the dual of this problem. First we write the Lagrangian:
\begin{align}
    &L(v,z_1,z_2,Z_1,Z_2,\rho_1,\rho_2) = \nonumber \\
    &=-\ell^*(-v) - \sum_{i=1}^n v_i x_i^T (Z_1 - Z_2)x_i + \beta \tr(Z_1 + Z_2) - \sum_{j=1}^d (Z_{1,jj}z_{1,j} + Z_{2,jj}z_{2,j}) + \sum_{j=1}^d (\rho_1 z_{1,j} + \rho_2 z_{2,j})
\end{align}
where $Z_1,Z_2 \in \mathbb{S}^{d\times d}$ and $\rho_1,\rho_2 \in \mathbb{R}$ are the Lagrange multipliers. Maximizing the Lagrangian with respect to $v,z_1,z_2$ leads to the following convex program
\begin{align}
    \min_{Z_1, Z_2, \rho_1, \rho_2} \quad & \ell \left(\begin{bmatrix} x_1^T (Z_1 - Z_2)x_1 \\ \vdots \\ x_n^T (Z_1 - Z_2)x_n \end{bmatrix}, \, y \right) + \beta \tr(Z_1 + Z_2) \nonumber \\
    \mbox{s.t.} \quad & Z_{1,jj} = \rho_1, \, Z_{2,jj} = \rho_2, \,\, j=1,\dots,d \nonumber \\
    &Z_1 \succeq 0, \, Z_2 \succeq 0 \,.
\end{align}
Equivalently,
\begin{align}
    p^*\ge d_{\mathrm{SDP}}:=\min_{Z_1, Z_2, \rho_1, \rho_2} \quad & \ell \left(\hat{y}, \, y \right) + \beta d (\rho_1 + \rho_2) \nonumber \\
    \mbox{s.t.} \quad & \hat{y}_i = x_i^T(Z_1-Z_2)x_i, \, i=1,\dots,n \nonumber \\
    & Z_{1,jj} = \rho_1, \, Z_{2,jj} = \rho_2, \, j=1,\dots,d \nonumber \\
    &Z_1 \succeq 0, \, Z_2 \succeq 0 \,.
\end{align}
Remarkably, the above SDP is a polynomial time tractable lower bound for the combinatorial non-convex problem $p^*$. 

\subsection{SDP relaxation for bilinear activation networks}
Now we focus on the bilinear architecture $f(x)=\sum_{j=1}^m (x^T u_j) (x^T v_j) \alpha_j$ and provide an SDP relaxation for the corresponding non-convex training problem. It will be shown that the resulting SDP relaxation is provably tight, where a matching upper bound can be obtained via randomization. Moreover, the resulting feasible solutions can be transformed into a quantized neural network with polynomial activations as we have shown in Section \ref{sec:lifting}.
Consider the non-convex training problem for the two-layer network with the bilinear activation given by
\begin{align}
    p^*_b = &\min_{\mbox{s.t.}\, u_j,v_j\in\{-1,1\}^d,\alpha_j\in\mathbb{R}\,\forall j\in[m]}  \ell \left(\sum_{j=1}^m ((Xu_j)\circ (Xv_j)) \alpha_j, \, y \right) + \beta d \sum_{j=1}^m \vert\alpha_j\vert \,.
    \label{eq:nonconvexbilinear}
\end{align}
Here $\circ$ denotes the Hadamard, i.e., direct product of two vectors. Repeating an analogous duality analysis (see Section \ref{sec:dual_analysis_bilinear} for the details), we obtain a tractable lower-bounding problem given by
\begin{align} \label{eq:sdpbilinear}
  p_b^* \ge d_{\mathrm{bSDP}}:= \min_{Q, \rho} \quad & \ell \left(\hat{y}, \, y \right) + \beta d \rho \nonumber \\
    \mbox{s.t.} \quad & \hat{y}_i = 2x_i^TZx_i, \, i=1,\dots,n \nonumber \\
    & Q_{jj} = \rho, \, j=1,\dots,2d\nonumber \\
    &Q=\begin{bmatrix} V & Z \\ Z^T & W \end{bmatrix} \succeq 0 \,.
\end{align}
The above optimization problem is a convex SDP, which can be solved efficiently in polynomial time.

%% file: sampling_algorithm.tex
\section{Main result: SDP Relaxation is Tight} \label{sec:sampling_algorithm}
We first introduce an existence result on covariance matrices which will be used in our quantized neural network construction. 
\begin{theorem}[Trigonometric covariance shaping]
\label{thm:cov}
Suppose that $Z^* \in \mathbb{R}^{d\times d}$ is an arbitrary matrix such that $\exists V,W:\, \begin{bmatrix} V & Z^* \\ Z^{*T} & W \end{bmatrix} \succeq 0$ and $V_{jj}=W_{jj}=1\,\forall j$. Then, there exists a PSD matrix $Q \in \mathbb{R}^{2d\times 2d} \succeq 0$ satisfying $Q_{jj}=1\,\forall j$ and
\begin{align}
    \arcsin( Q_{(12)}) = \gamma Z^*
\end{align}
where $Q=\begin{bmatrix} Q_{(11)} & Q_{(12)} \\ Q_{(21)} & Q_{(22)} \end{bmatrix}$, $\gamma=\ln(1+\sqrt{2})$, and $\arcsin$ is the elementwise inverse sine function.
\end{theorem}

Our construction is based on randomly generating quantized neurons whose empirical covariance matrix is matched to the optimal solution of the convex SDP. The above theorem is an existence result which will be crucial in our sampling algorithm. The important observation is that, if we let $\begin{bmatrix} u \\ v \end{bmatrix} \sim \sign( \mathcal{N}(0,Q))$ with some $Q\succeq 0,\, Q_{jj}=1\,\forall j$, then $\Exs \begin{bmatrix} u \\ v \end{bmatrix} \begin{bmatrix} u \\ v \end{bmatrix}^T  = \frac{2}{\pi} \arcsin(Q)$, which is referred to as Grothendieck's Identity  \cite{alon2004approximating}. Therefore, $\Exs [uv^T] = \arcsin Q_{(12)} = \gamma Z^*$, which is proportional to the target covariance matrix. This algorithm is inspired by Krivine's analysis of the Grothendieck's constant and its applications in approximating the cut norm using semidefinite programming \cite{alon2004approximating}. 
\begin{proof}[Proof of Theorem \ref{thm:cov}]
Note that the condition $\exists V,W:\, \begin{bmatrix} V & Z^* \\ Z^{*T} & W \end{bmatrix} \succeq 0$ and $V_{jj}=W_{jj}=1\,\forall j$ implies that there exists unit norm vectors $x_1,\dots,x_d,y_1,\dots,y_d$ such that $Z^*_{ij}=x_i^Ty_j$. Consequently, applying Lemma 4.2 of \cite{alon2004approximating} and Grothendieck's Identity completes the proof.
\end{proof}

\subsection{Sampling algorithm for approaching the global optimum} \label{subsec:sampling_alg}
Now we present our sampling algorithm which generates quantized neural networks parameters based on the solution of the lower-bounding convex SDP. The algorithm is listed in Algorithm \ref{alg:sampling_alg}. We explain each step of the algorithm below.

\begin{algorithm}[tb]
  \caption{Sampling algorithm for quantized neural networks}
  \label{alg:sampling_alg}
  \begin{enumerate}
  \item Solve the SDP in \eqref{eq:sdpbilinear}.  Define the scaled matrix $Z_s^* \leftarrow Z^*/\rho^*$.
  \item Solve the problem
    \begin{align} \label{eq:cov_matrix_Q_problem}
        Q^* := \arg \min_{Q\succeq 0, Q_{jj}=1\forall j} \|Q_{(12)} -  \sin(\gamma Z_s^*) \|_F^2 \,.
    \end{align}
  \item Sample the first layer weights $u_1,\dots,u_m, v_1,\dots,v_m$ from multivariate normal distribution as $\begin{bmatrix} u \\ v \end{bmatrix} \sim \sign( \mathcal{N}(0,Q^*))$ and set the second layer weights as $\alpha_j= \rho^* \frac{\pi}{\gamma m},\,\forall j$.
  \item (optional) Transform the quantized bilinear activation network to a quantized polynomial activation network.
  \end{enumerate}
\end{algorithm}

\begin{enumerate}
    \item Solve the SDP in \eqref{eq:sdpbilinear} to minimize the training loss. Denote the optimal solution as $Z^*$ and $\rho^*$ and define the scaled matrix $Z_s^* \leftarrow Z^*/\rho^*$.
    
    \item Find the $2d\times 2d$ covariance matrix $Q^*$ by solving \eqref{eq:cov_matrix_Q_problem} with $Q=\begin{bmatrix} Q_{(11)} & Q_{(12)} \\ Q_{(21)} & Q_{(22)} \end{bmatrix}$ where the notation $Q_{(ij)}$ denotes a $d\times d$ block matrix. $\gamma=\ln(1+\sqrt{2})$, and $\sin(\cdot)$ is the element-wise sine function. The objective value is guaranteed to be zero due to Theorem \ref{thm:cov}. Therefore we have $\arcsin(Q^*_{(12)})=\gamma Z_s^*$ and $Q^*\succeq 0, Q^*_{jj}=1\,\forall j$. 
    
    \item Sample $u_1,\dots,u_{m}, v_1,\dots,v_{m}$ via $\begin{bmatrix} u \\ v \end{bmatrix} \sim \sign( \mathcal{N}(0,Q^*))$. Since $\Exs[uv^T] = \frac{2}{\pi}\arcsin Q^*_{(12)} = \frac{2\gamma}{\pi} Z_s^*$ as a corollary of Theorem \ref{thm:cov}, we have $\Exs[\frac{1}{m} \sum_{j=1}^{m} u_j v_j^T] = \frac{2\gamma}{\pi} Z_s^*$. We set $\alpha_j= \rho^* \frac{\pi}{\gamma m},\,\forall j$ to obtain $\Exs [\sum_{j=1}^{m} u_j v_j^T\alpha_j] = 2Z_s^* \rho^*=2Z^*$. This is as desired since the SDP computes the predictions via $\hat{y}_i=2x_i^TZx_i$.
    
    
    \item This optional step can be performed as described in Section \ref{sec:lifting}.
\end{enumerate}

The extension of the sampling algorithm to the vector output networks is given in Section \ref{sec:vector_output}. 

\subsection{Concentration around the mean}
We establish a probabilistic bound on the convergence of the empirical sum $\frac{1}{m} \sum_{j=1}^m u_jv_j^T$ in the step 3 of the sampling algorithm to its expectation. Our technique involves applying Matrix Bernstein concentration bound for sums of i.i.d. rectangular matrices \cite{tropp2015introduction} to obtain:
%
%
\begin{align}
    \mathbb{P} \left[ \left \| \frac{1}{m }\sum_{j=1}^m u_jv_j^T - \Exs [u_1v_1^T] \right\|_2 \geq \epsilon \right] &\leq \exp\left( - \frac{m\epsilon^2}{(2\gamma / \pi)^2 \|Z_s^*\|_2^2 + d(c^\prime + 2\epsilon/3)}+\log(2d)\right).
\end{align}
for all $\epsilon>0$. 

We summarize this analysis in the following theorem, which is our main result.
%

\begin{theorem}[Main result]
\label{thm:main_v2}
Suppose that the number of neurons satisfies $m \geq c_1 \frac{L_c^2 R_m^4 d\log(d)}{\epsilon^2}$. Let $L_c$ denote the Lipschitz constant of the vectorized loss function under the $\ell$-infinity norm, i.e. $|\ell(z)-\ell(z^\prime)|\le L_{c}\|z-z^\prime\|_{\infty}$, and define $R_m :=\max_{i\in[n]}\|x_i\|_2$. Then, Algorithm \ref{alg:sampling_alg} generates a quantized neural network with weights $\hat u_j,\hat v_j \in \{-1,+1\}^d$ and $\hat \alpha_j=\frac{\rho^*\pi}{m\log(1+\sqrt{2}) }$, $j=1,\dots,m$ that achieve near optimal loss, i.e.,
\begin{align}
    \Big | \ell \big(\sum_{j=1}^m ((X\hat u_j) \circ (X \hat v_j)) \hat \alpha_j, \, y \big) - \ell \big(\sum_{j=1}^m ((X u^*_j) \circ (Xv^*_j)) \alpha^*_j, \, y \big) \Big| \leq \epsilon 
\end{align}
with probability at least $1-c_2e^{-c_3 \epsilon^2 m/d}$ for certain constants $c_1,c_2,c_3$ when the regularization coefficient satisfies $\beta \leq \frac{\epsilon}{d} \min\left( \frac{1}{\sum_j |\hat{\alpha}_j|}, \frac{1}{\sum_j |\alpha^*_j|} \right)$. The weights $u_j^*, v_j^* \in \{-1,+1\}^d, \alpha^*_j \in \mathbb{R}$, $j=1,\dots,m$ are the optimal network parameters for the non-convex combinatorial problem in  \eqref{eq:nonconvexbilinear}.
\end{theorem}

\begin{remark}
For loss functions of the form $\ell(z)=\frac{1}{n}\sum_{i=1}^n \phi(z_i)$, where $\phi(\cdot)$ is a scalar $L_{c}$-Lipschitz loss satisfying $|\phi(s)-\phi(s^\prime)|\le L_{c} |s-s^\prime|$, the vectorized loss function $\ell(z)$ is $L_c$-Lipschitz under the infinity norm. This fact follows from $\big|\frac{1}{n}\sum_{i=1}^n \phi(z_i)-\frac{1}{n}\sum_{i=1}^n \phi(z^\prime_i) \big|\le \frac{1}{n}\sum_{i=1}^n L_c |z_i-z_i^\prime|\le L_c \|z-z^\prime \|_{\infty}$. Examples of $1$-Lipschitz loss functions include hinge loss, logistic loss and $\ell_1$ loss, which satisfy our assumption with $L_c=1$.
\end{remark}
\begin{remark}
Our main result also holds when $\beta\rightarrow 0$. In this regime, the constraint $\beta \leq \frac{\epsilon}{d} \min\left( \frac{1}{\sum_j |\hat{\alpha}_j|}, \frac{1}{\sum_j |\alpha^*_j|} \right)$ is always satisfied.
\end{remark}

The proof of Theorem \ref{thm:main_v2} is provided in Section \ref{sec:proof_main_thm}. To the best of our knowledge, this is the first result on polynomial-time optimal trainability of quantized neural networks.
We remark that one can transform the near optimal quantized bilinear activation network to a near optimal quantized polynomial activation network with the mapping shown in Section \ref{sec:lifting}. Consequently, this result also applies to approximating the solution of \eqref{eq:nonconvexpolyact}.

Additionally, note that the second layer weights are all identical, which allows us to represent the sampled neural network using $2md$ bits and only one scalar floating point variable. One can employ the reduction in Section \ref{subsec:multilevel} to train optimal multi-level quantized neural networks using the above result in polynomial time. 

Furthermore, it is interesting to note that, overparameterization is a key component in enabling optimization over the combinatorial search space of quantized neural networks in polynomial time. In contrast, the problems in \eqref{eq:nonconvexbilinear} and \eqref{eq:nonconvexpolyact} are NP-hard when $m=1$.

%% file: numerical_results.tex
\section{Numerical Results} \label{sec:numerical_results}

In this section, we present numerical results that verify our theoretical findings. Additional numerical results can be found in the Appendix.

We compare the performance of the proposed SDP based method against a backpropagation based method that we describe in the next subsection. We have used CVXPY \cite{diamond2016cvxpy,agrawal2018rewriting} for solving the convex SDP. In particular, we have used the open source solver SCS (splitting conic solver) \cite{scs2016paper,scs2016code} in CVXPY, which is a scalable first order solver for convex cone problems. Furthermore, in solving the non-convex neural network training problems that we include for comparison, we have used the stochastic gradient descent (SGD) algorithm with momentum in PyTorch \cite{pytorch}. 

The experiments have been carried out on a MacBook with 2.2 GHz 6-Core Intel Core i7 processor and 16 GB of RAM.

\subsection{Planted dataset experiment}
Figure \ref{fig:planted_exp_pos_sec_layer} shows the cost as a function of the number of neurons $m$. The neural network architecture is a two-layer fully connected network with bilinear activation, i.e., $f(x)=\sum_{j=1}^m (x^Tu_j)(x^Tv_j)\alpha_j$. This experiment has been done using a planted dataset. The plot compares the method described in Section \ref{sec:sampling_algorithm} against a backpropagation based quantization method.

\begin{figure} [ht]
\begin{minipage}[b]{0.48\textwidth}
\centering
  \centerline{\includegraphics[width=0.9\columnwidth]{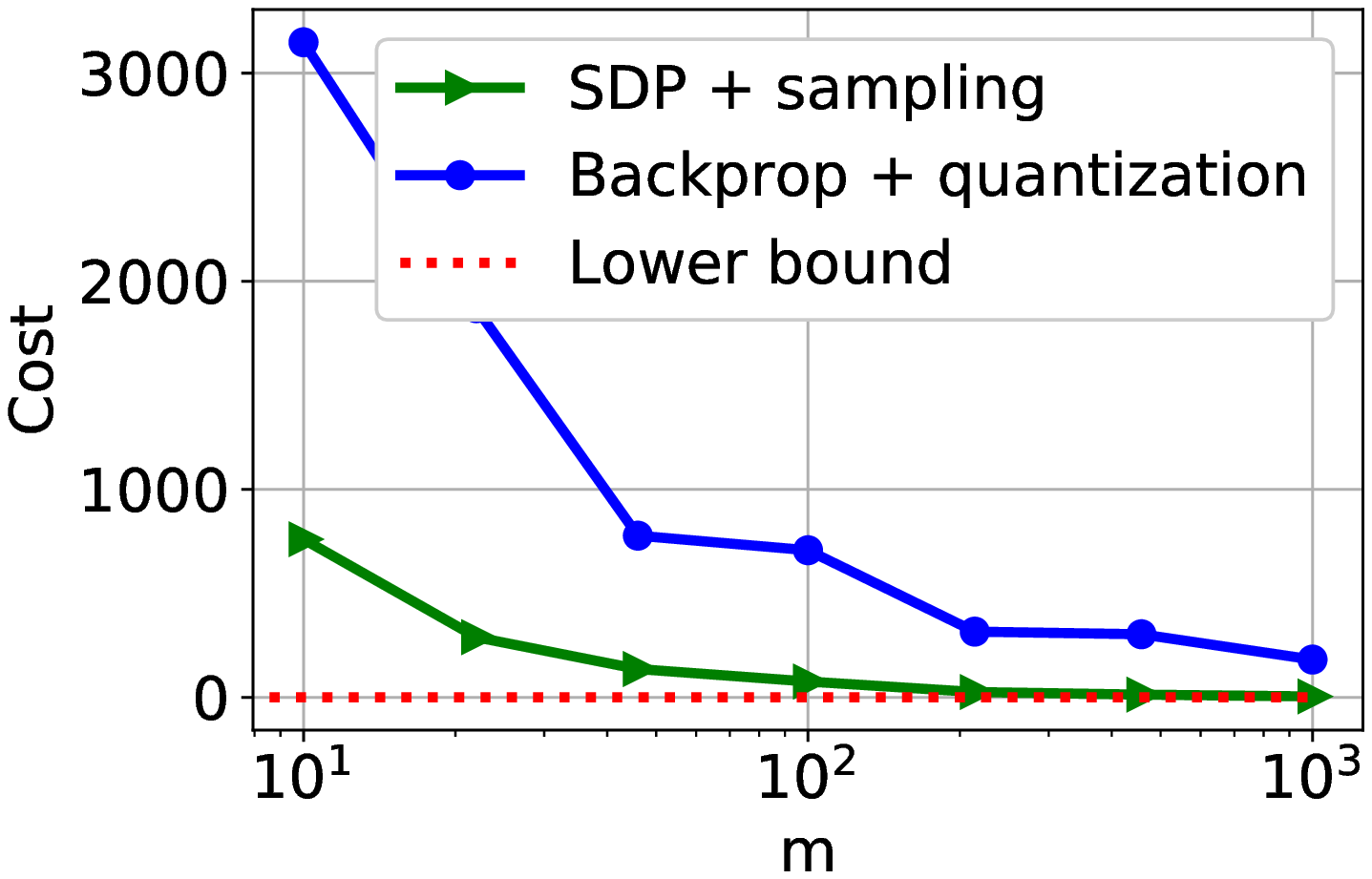}}
  \centerline{a) Training error}
\end{minipage}
\hfill
\begin{minipage}[b]{0.48\textwidth}
\centering
  \centerline{\includegraphics[width=0.9\columnwidth]{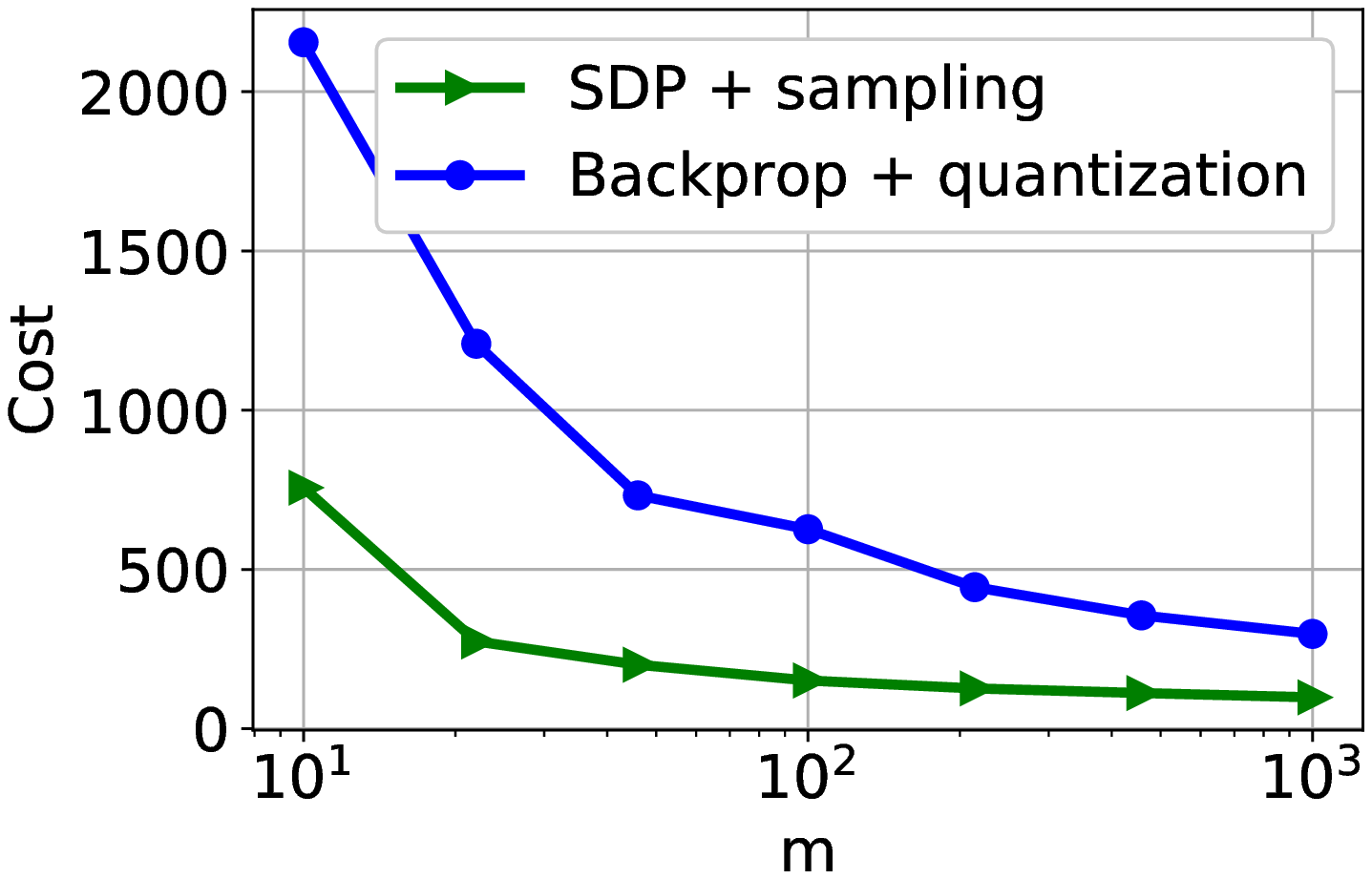}}
  \centerline{b) Test error}
\end{minipage}
\caption{Objective against the number of neurons $m$. Dataset $X$ has been synthetically generated via sampling from standard Gaussian distribution and has dimensions $n=100$, $d=20$. The target vector $y$ has been computed via a planted model with $10$ planted neurons. In the planted model, the first layer weights are binary and the second layer weights are real and non-negative. The regularization coefficient is $\beta=10^{-4}$. The lower bound is obtained by solving the SDP in Section \ref{sec:duality}. Plots a and b show the cost on the training and test sets, respectively. The test set has been generated synthetically by sampling from the same distribution as the training set.
}
\label{fig:planted_exp_pos_sec_layer}
\end{figure}

The algorithm in Section \ref{sec:sampling_algorithm} solves the relaxed SDP and then samples binary weights as described previously. This procedure results in $2md$ binary weights for the first layer. The second layer weights are all equal to $\rho^* \pi / (\gamma m)$. This network requires storage of $2md$ bits and a single real number. Furthermore, post-training quantization using backpropagation works as follows. First, we train a two-layer neural network with bilinear activation in PyTorch \cite{pytorch} with $m$ neurons using stochastic gradient descent (SGD). We fix the second layer weights to $1/m$ during training. After training, we form the matrices $\hat{Z} = \sum_{j=1}^m \sign(u_j)\sign(v_j^T)$ and $Z^* = \sum_{j=1}^m u_jv_j^T \frac{1}{m}$. Then, the solution of the problem $\min_{c\in \mathbb{R}} \| c\hat{Z} - Z^*\|_F^2$ is used to determine the second layer weights as $c$. The optimal solution is given by $c=\frac{\langle \hat{Z}, Z^*\rangle}{\langle Z^*, Z^* \rangle}$. This procedure results in $2md$ bits for the first layer and a single real number for the second layer and hence requires the same amount of storage as the SDP based method. 

In addition to low storage requirements, this particular network is very efficient in terms of computation. This is very critical for many machine learning applications as this translates to shorter inference times. For the two-layer neural network with bilinear activation, the hidden layer computations are $2md$ additions since the weights are $\{+1,-1\}$ and the bilinear activation layer performs $m$ multiplications (i.e. $(x^Tu_j)(x^Tv_j)$ $j=1,\dots,m$). The second layer requires only $m$ additions and one multiplication since the second layer weights are the same.

Figure \ref{fig:planted_exp_pos_sec_layer} shows that the SDP based method outperforms the backpropagation approach. Also, we observe that the cost of the SDP based method approaches the lower bound rapidly as the number of neurons $m$ is increased. Furthermore, plot b shows that the test set performance for the SDP based method is also superior to the backpropagation based method.

We note that another advantage of the SDP based sampling method over backpropagation is that we do not need to solve the SDP for a fixed number of neurons $m$. That is, the SDP does not require the number of neurons $m$ as an input. The number of neurons is used only during the sampling process. This enables one to experiment with multiple values for the number of neurons without re-solving the SDP.

\subsection{Real dataset experiment}
Figure \ref{fig:realdata_exp} compares the backpropagation approach and the SDP based method on a real dataset from UCI machine learning repository \cite{uci2019datasets}. The dataset is the binary classification "breast-cancer" dataset and has $n=228$ training samples and $58$ test samples and the samples are $d=9$ dimensional. Figure \ref{fig:realdata_exp} shows the classification accuracy against time for various methods which we describe below. The regularization coefficient $\beta$ is picked for each method separately by searching the value that yields the highest accuracy and the resulting $\beta$ values are provided in the captions of the figures.

\begin{figure} [ht]
\begin{minipage}[b]{0.48\textwidth}
\centering
  \centerline{\includegraphics[width=0.9\columnwidth]{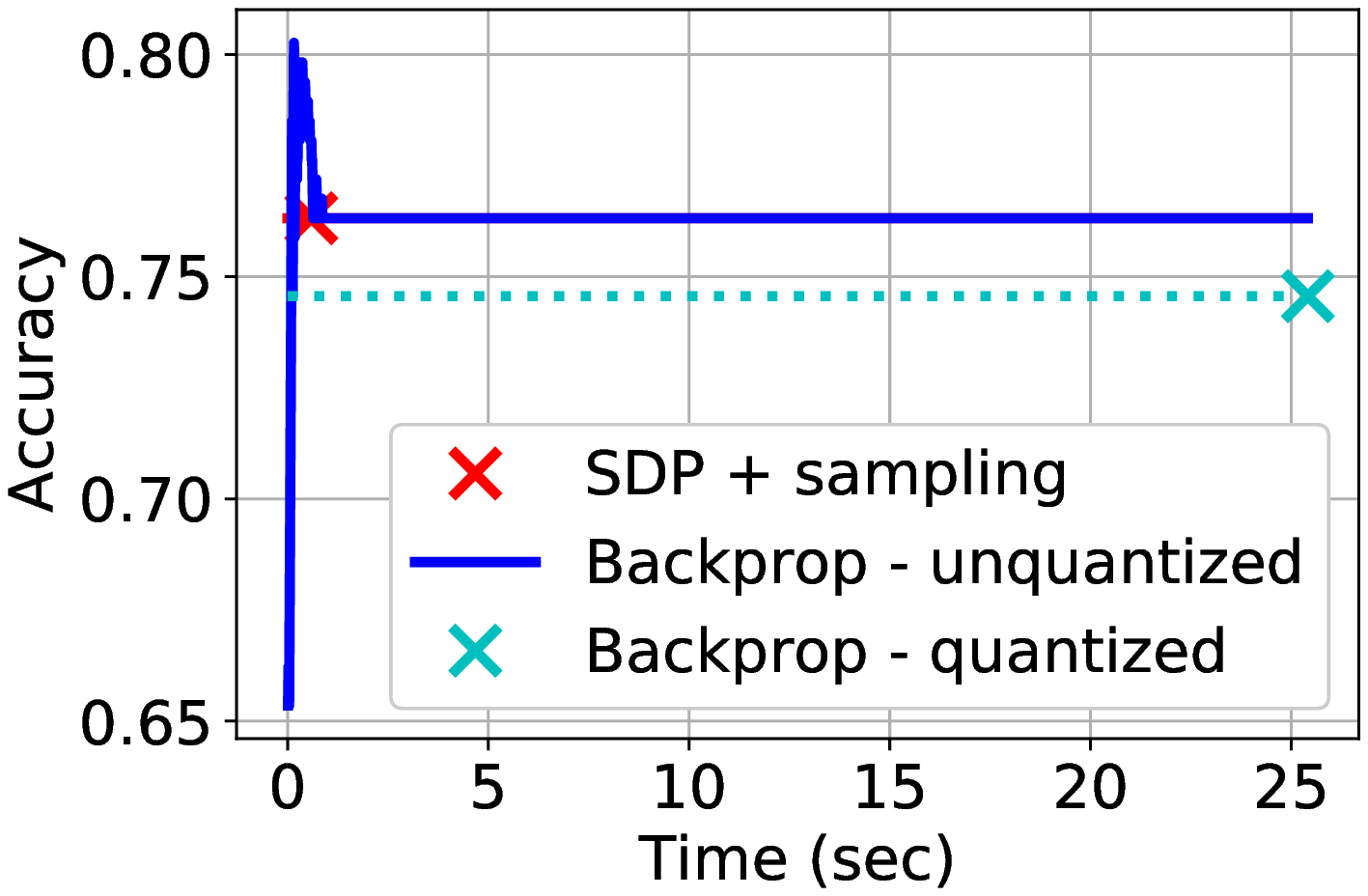}}
  \centerline{a) Training accuracy}
\end{minipage}
\hfill
\begin{minipage}[b]{0.48\textwidth}
\centering
  \centerline{\includegraphics[width=0.9\columnwidth]{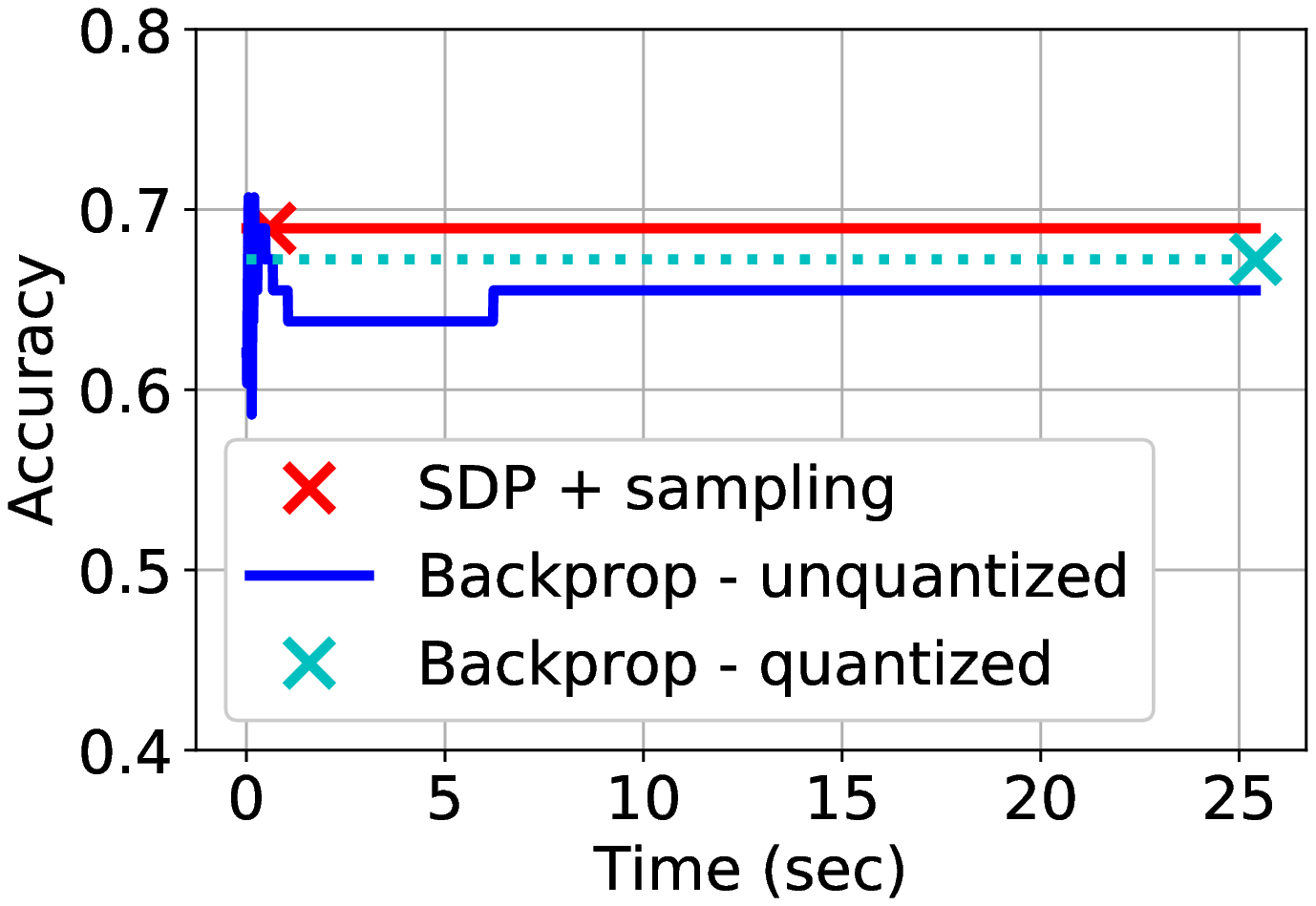}}
  \centerline{b) Test accuracy} 
\end{minipage}
\caption{Classification accuracy against wall-clock time. Breast cancer dataset with $n=228,d=9$. The number of neurons is $m=250$ and the regularization coefficient is $\beta=0.1$ for the SDP based method and $\beta=0.1$ for the backpropagation.}
\label{fig:realdata_exp}
\end{figure}

Figure \ref{fig:realdata_exp} shows the training and test accuracy curves for backpropagation without quantization by the blue solid curve. After the convergence of the backpropagation, we quantize the weights as described in the previous subsection, and the timing and accuracy for the quantized model are indicated by the cyan cross marker. The timing and accuracy of the SDP based method are shown using the red cross marker. Figure \ref{fig:realdata_exp} demonstrates that the SDP based method requires less time to return its output. We observe that quantization reduces the accuracy of backpropagation to a lower accuracy than the SDP based method's accuracy.

It is important to note that in neural network training, since the optimization problems are non-convex, it takes considerable effort and computation time to determine the hyperparameters that will achieve convergence and good performance. For instance, among the hyperparameters that require tuning is the learning rate (i.e. step size). We have performed the learning rate tuning for the backpropagation algorithm offline and hence it is not reflected in Figure \ref{fig:realdata_exp}. Remarkably, the proposed convex SDP based method does not require this step as it is readily handled by the convex SDP solver.

Figure \ref{fig:realdata_exp_ionosphere} shows results for the UCI repository dataset "ionosphere". This is a binary classification dataset with $n=280$ training samples and $71$ test samples. The samples are $d=33$ dimensional. The experiment setting is similar to Figure \ref{fig:realdata_exp} with the main difference that the number of neurons is $10$ times higher (i.e., $m=2500$). We observe that the SDP based method outperforms the quantized network trained with backpropagation on both training and test sets.

\begin{figure} [ht]
\begin{minipage}[b]{0.48\textwidth}
\centering
  \centerline{\includegraphics[width=0.9\columnwidth]{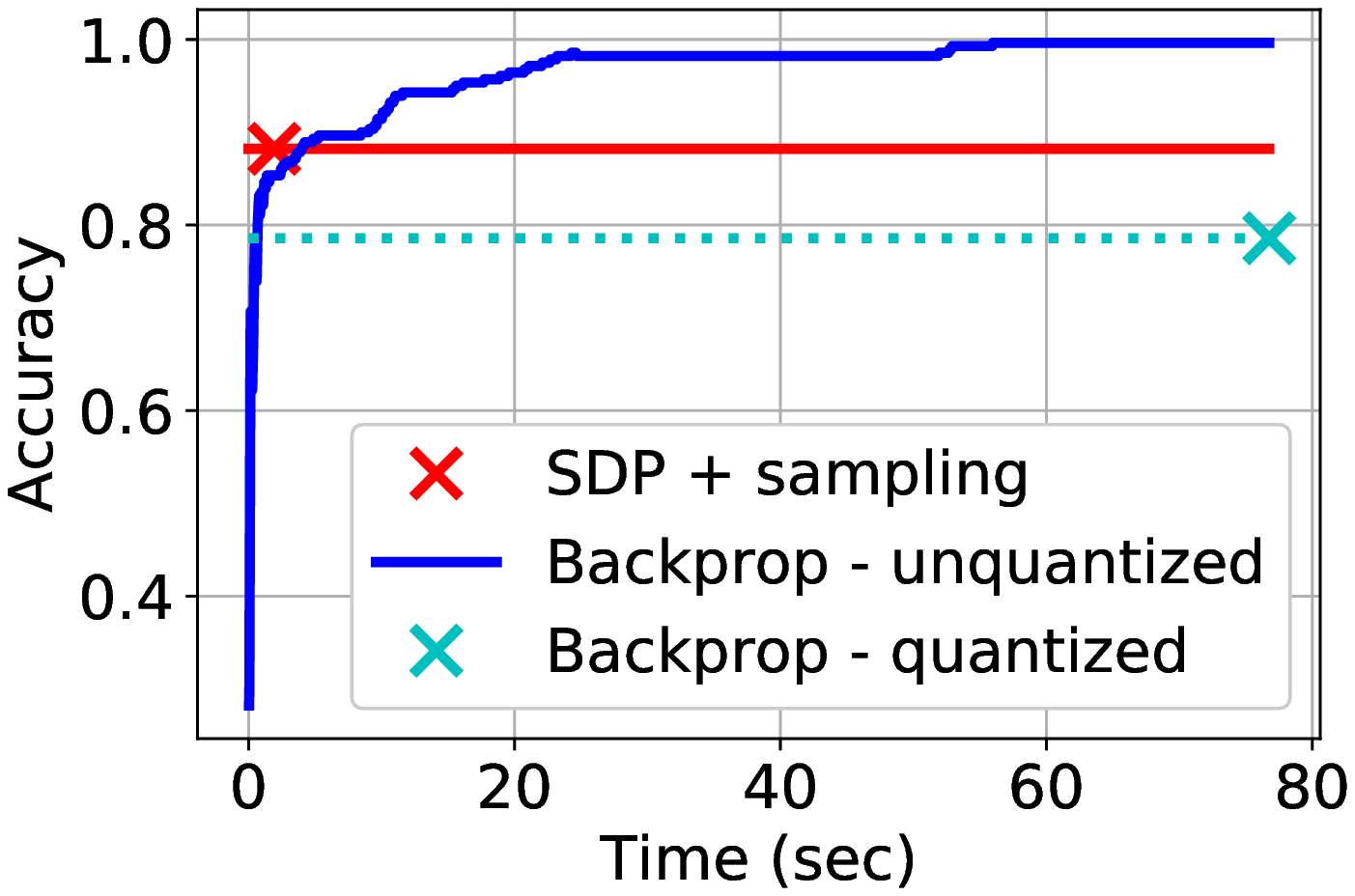}}
  \centerline{a) Training accuracy}
\end{minipage}
\hfill
\begin{minipage}[b]{0.48\textwidth}
\centering
  \centerline{\includegraphics[width=0.9\columnwidth]{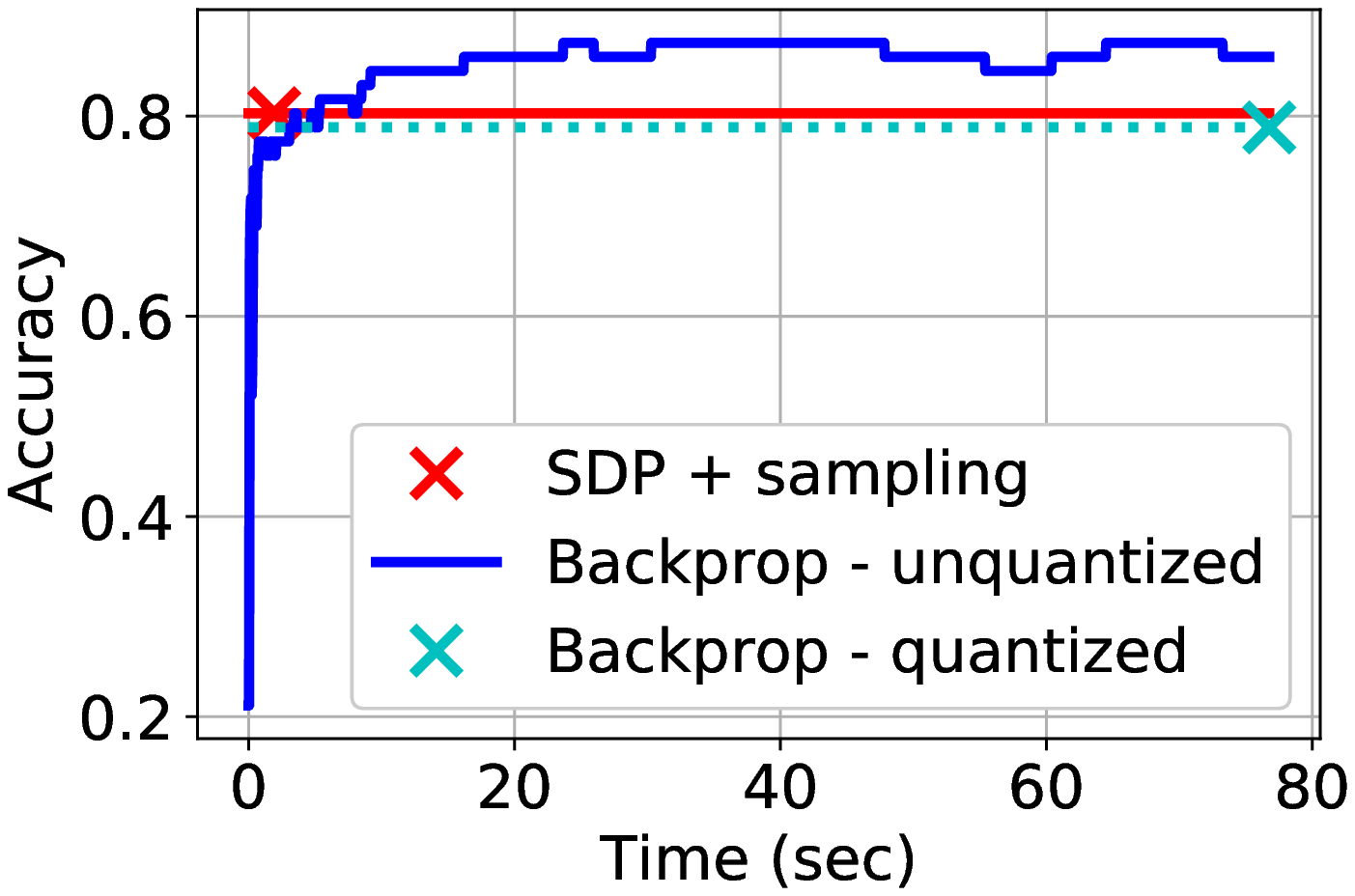}}
  \centerline{b) Test accuracy}
\end{minipage}
\caption{Classification accuracy against wall-clock time. Ionosphere dataset with $n=280,d=33$. The number of neurons is $m=2500$ and the regularization coefficient is $\beta=10$ for the SDP based method, $\beta=10^{-6}$ for backpropagation.}
\label{fig:realdata_exp_ionosphere}
\end{figure}

%% file: conclusion.tex
\section{Conclusion} \label{sec:conclusion}

We introduced a convex duality based approach for training optimal quantized neural networks with degree two polynomial activations. We first presented a lower-bounding semidefinite program which is tractable in polynomial time. We also introduced a bilinear activation architecture, and the corresponding SDP lower-bound. We showed that bilinear architectures with binary quantization are sufficient to train optimal multi-level quantized networks with polynomial activations. We presented a sampling algorithm to generate quantized neural networks using the SDP by leveraging Grothendieck's identity and the connection to approximating the cut norm. Remarkably, we showed that mild overparameterization is sufficient to obtain a near-optimal quantized neural network via the SDP based sampling approach. Numerical experiments show that our method can generate significantly more accurate quantized neural networks compared to the standard post-training quantization approach. Moreover, the convex optimization solvers are faster than backpropagation in small to medium scale problems.

An immediate open question is to extend our results to deeper networks and different architectures, such as ReLU networks. For instance, our algorithm can be applied with polynomial approximations of ReLU. Moreover, one can apply our algorithm layerwise to optimally quantize a pre-trained neural network by knowledge distillation.

We acknowledge that our current numerical results are limited to small and medium datasets due to the memory constraints of standard SDP solvers. However, one can design custom optimization methods to obtain approximate solutions of the SDP for larger dimensional instances. The SDPs can also be defined and solved in deep learning frameworks with appropriate parameterizations. Random projection and sketching based optimizers for high-dimensional convex programs \cite{yurtsever2021scalable,yurtsever2017sketchy, lacotte2020adaptive} and randomized preconditioning \cite{lacotte2020limiting, lacotte2020optimal, lacotte2020effective, ozaslan2019iterative, lacotte2021fast} can address these computational challenges.  We leave this as an important open problem.

From a complexity theoretical perspective, it is remarkable that overparameterization breaks computational barriers in combinatorial and non-convex optimization. Specifically, it is straightforward to show that training a quantized neural network when $m=1$, i.e., a single neuron is NP-hard due to the connection to the MaxCut problem. However, allowing $m=\mathcal{O}(d\log d)$ enables optimization over a combinatorial search space in polynomial time. Exploring the other instances and limits of this phenomenon is another interesting research direction.

\section*{Acknowledgments}
This work was partially supported by the National Science Foundation under grants IIS-1838179 and ECCS-2037304, Facebook Research, Adobe Research and Stanford SystemX Alliance.

%% file: supp_file.tex
\newpage
\appendix
\onecolumn


\section{Proofs}

\subsection{Proof of Theorem \ref{thm:reduction_to_binary}}
\begin{proof}
First we show that the multiplication of the input $x \in \mathbb{R}^d$ by a multi-level quantized weight vector $q \in \mathcal{Q}_{M}^d$ can be represented by the dot product of a function of the input, i.e., $\tilde{x}$ and a binary quantized weight vector $u$, that is, $q^Tx=u^T \tilde{x}$. Here, $u$ is a binary vector of size $dM$ with entries satisfying
\begin{align} \label{eq:q_i_summation}
    q_i := \sum_{k=1}^M u_{k+(i-1)M},\, i=1,\dots,d \,.
\end{align}
For instance, for $M=4$, we have $q_1=u_1+u_2+u_3+u_4$. Note that because $u_j$'s are from the set $\{-1,+1\}$, we have that $q_1 \in \{-4,-2,0,2,4\}$, which is equal to the set for $(4+1=5)$-level quantization, i.e., $\mathcal{Q}_4$. The second entry of the $q$ vector similarly satisfies $q_2=u_5+u_6+u_7+u_8 \in \mathcal{Q}_4$. The same holds for all the entries $q_1,\dots,q_d$. 

Next, plugging in \eqref{eq:q_i_summation} in the dot product $q^Tx$ yields
\begin{align}
    q^Tx &= \sum_{i=1}^d q_i x_i = \sum_{i=1}^d \sum_{k=1}^M u_{k+(i-1)M} x_i \nonumber \\ 
    &= \sum_{i=1}^d \sum_{k=1}^M u_{k+(i-1)M} \tilde{x}_{k+(i-1)M} \nonumber \\
    &= u^T \tilde{x}
\end{align}
where we defined $\tilde{x} := \begin{bmatrix} x_1,x_1,\dots,x_1,x_2,x_2,\dots,x_2,\dots,x_d,x_d,\dots,x_d \end{bmatrix}^T \in \mathbb{R}^{dM}$. This shows that the dot product $q^Tx$ is equal to the dot product $u^T \tilde{x}$ where $u$ is a $dM$-dimensional vector with binary entries.

The input-output relationship for the two-layer fully connected neural network with polynomial activation is $f(x) = \sum_{j=1}^m \sigma(x^Tq_j) \alpha_j = \sum_{j=1}^m \left( aq_j^T xx^Tq_j + bq_j^Tx + c \right)\alpha_j$ where $q_j \in \mathcal{Q}_M^d$ and $\alpha_j \in \mathbb{R}$, $j=1,\dots,m$. Using the fact that we can represent a dot product with multi-level quantized weights as a dot product with binary quantized weights, we equivalently have
\begin{align}
    f(x) = \sum_{j=1}^m \left( au_j^T \tilde{x}\tilde{x}^Tu_j + bu_j^T\tilde{x} + c \right)\alpha_j \,.
\end{align}
We can rewrite this as a neural network with quadratic activation:
\begin{align}
    f(x) &= \sum_{j=1}^m \begin{bmatrix} u_j^T & 1 \end{bmatrix} \begin{bmatrix} a\tilde{x}\tilde{x}^T & \frac{b}{2}\tilde{x} \\ \frac{b}{2}\tilde{x}^T & c \end{bmatrix} \begin{bmatrix} u_j \\ 1 \end{bmatrix} \alpha_j \nonumber \\
    &= \sum_{j=1}^m \tilde{u}_j^T \mathcal{X} \tilde{u}_j \alpha_j
\end{align}
where we have defined $\tilde{u}_j \in \{-1,+1\}^{dM+1}$, $j=1,\dots,m$, and $\mathcal{X} \in \mathbb{R}^{(dM+1)\times (dM+1)}$. 

This representation can be seen as a bilinear activation network with $u_j^\prime=u_j$ and $v_j^\prime=u_j$, $j=1,\dots,m$. The proof of the converse follows from the symmetrization identity \eqref{eq:symmetrization_identity}. 
\end{proof}

\subsection{Proof of Theorem \ref{thm:main_v2}} \label{sec:proof_main_thm}

\begin{proof}
We begin by applying the matrix
Bernstein concentration bound on the matrices $(u_jv_j^T - \Exs[u_jv_j^T])$, $j=1,\dots,m$, which we note are $(d\times d)$-dimensional zero-mean i.i.d. matrices. We obtain the following upper bound on the spectral norm of these matrices
\begin{align}
    \| u_jv_j^T - \Exs[u_jv_j^T] \| &\leq \|u_jv_j^T\|_2 + \|\Exs[u_jv_j^T]\|_2 \nonumber \\
    &\leq \|u_jv_j^T\|_2 + \Exs[\|u_jv_j^T\|_2] \nonumber \\
    &= \|u_j\|_2 \|v_j\|_2 + \Exs[\|u_j\|_2 \|v_j\|_2] \nonumber \\
    &\leq d + d = 2d\,,
\end{align}
for $j = 1,\dots, m$ where we use the triangle inequality in the first line and Jensen's inequality in the second line. Next, we define $S_j := u_jv_j^T - \Exs[u_jv_j^T]$ and $S := \sum_{j=1}^m S_j$, then the matrix variance of the sum (which we will plug in the matrix concentration bound formula) is given by
\begin{align}
    \sigma^2 &= \max\{ \| \Exs[ SS^T ] \|_2 , \| \Exs[ S^TS ] \|_2\} = \max\left\{ \left\| \sum_{j=1}^m \Exs[ S_jS_j^T ] \right\|_2 , \left\| \sum_{j=1}^m \Exs[ S_j^TS_j ] \right\|_2 \right\}
\end{align}
where the second equality follows because $S_j$'s are zero-mean.
\begin{align}
    \Exs[S_jS_j^T] &= \Exs\left[\left(u_jv_j^T - \Exs[u_jv_j^T] \right) \left(u_jv_j^T - \Exs[u_jv_j^T] \right)^T\right] \nonumber \\
    &= d \Exs[u_ju_j^T] - \Exs[u_jv_j^T]\Exs[v_ju_j^T] \nonumber \\
    &= d \Exs[u_ju_j^T]- (2\gamma / \pi)^2 Z_s^*{Z_s^*}^T \nonumber \\
    &= d \Exs[u_ju_j^T]- (2\gamma / \pi Z_s^*)^2 \,.
\end{align}
Next, we bound the spectral norm of $\Exs[ SS^T ]$ as
\begin{align}
    \| \Exs[ SS^T ] \|_2 = \left\| \sum_{j=1}^m \Exs[ S_jS_j^T ] \right\|_2 &= \left\|\sum_{j=1}^m \left(d \Exs[u_ju_j^T]- (2\gamma / \pi Z_s^*)^2\right) \right\|_2 \nonumber \\
    &= \left\| md \Exs[u_1u_1^T] - m(2\gamma / \pi Z_s^*)^2 \right\|_2 \nonumber \\
    &\leq md \left\|\Exs[u_1u_1^T]\right\|_2 + \left\|m(2\gamma / \pi Z_s^*)^2 \right\|_2 \nonumber \\
    &= md \left\|\Exs[u_1u_1^T]\right\|_2 + m(2\gamma / \pi)^2 \|Z_s^*\|_2^2 \nonumber \\
    &= md (2\gamma/\pi)\|\arcsin(Q_{(11)})\|_2 + m(2\gamma / \pi)^2 \|Z_s^*\|_2^2 \,.
\end{align}
The last line follows from the identity $\Exs[u_1u_1^T] = 2\gamma / \pi\arcsin(Q_{(11)})$.
We note that the upper bound for $\| \Exs[ SS^T ] \|_2$ is also an upper bound for $\| \Exs[ S^TS ] \|_2$. Hence, the matrix variance is upper bounded by $\sigma^2 \leq c^\prime md + m(2\gamma / \pi)^2 \|Z_s^*\|_2^2$ where $c^\prime \geq 0$ is a constant. Applying the matrix Bernstein concentration bound yields
\begin{align}
     \mathbb{P} \left[ \left \| \sum_{j=1}^m (u_jv_j^T - \Exs [u_jv_j^T]) \right\|_2 \ge m\epsilon \right] \leq 2d \exp\left(\frac{-m^2\epsilon^2}{\sigma^2 + 2dm\epsilon/3} \right) \,.
\end{align}
Plugging in the expression for the variance, we obtain
\begin{align}
    \mathbb{P} \left[ \left \| \frac{1}{m }\sum_{j=1}^m u_jv_j^T - \Exs [u_1v_1^T] \right\|_2 \ge \epsilon \right] &\leq 2d \exp\left(\frac{-m^2\epsilon^2}{c^\prime md + m(2\gamma / \pi)^2 \|Z_s^*\|_2^2 + 2dm\epsilon/3} \right) \nonumber \\
    &= 2d\exp\left( - \frac{m\epsilon^2}{(2\gamma / \pi)^2 \|Z_s^*\|_2^2 + d(c^\prime + 2\epsilon/3)}\right) \nonumber \\
    &= \exp\left( - \frac{m\epsilon^2}{(2\gamma / \pi)^2 \|Z_s^*\|_2^2 + d(c^\prime + 2\epsilon/3)}+\log(2d)\right).
\end{align}
Let us denote the optimal solution of the original non-convex problem as $Z_{nc}^*=\sum_{j=1}^m u_j^*(v_j^*)^T\alpha_j^*$ where the weights $u_j^*, v_j^* \in \{-1,+1\}^d, \alpha^*_j \in \mathbb{R}$, $j=1,\dots,m$ are optimal network parameters for the non-convex combinatorial problem in  \eqref{eq:nonconvexbilinear} . Solving the SDP gives us an unquantized solution $Z^*$ and via the sampling algorithm, we obtain the quantized solution given by $\hat{Z}=\sum_{j=1}^m \hat{u}_j \hat{v}_j^T\hat{\alpha}_j$. 

We now introduce some notation. We will denote the loss term in the objective by $L(Z)$ and the regularization term by $R(Z)$, that is,
\begin{align}
    L(Z) := \ell\left( \begin{bmatrix} x_1^TZx_1 \\ \vdots \\ x_n^TZx_n \end{bmatrix},\, y \right) , \quad R(Z) := d\sum_{j=1}^m |\alpha_j| \quad \mbox{when} \quad Z = \sum_{j=1}^m u_jv_j^T\alpha_j.
\end{align}
We now bound the difference between the losses for the unquantized solution of the SDP, i.e., $Z^*$, and the quantized weights $\hat{Z} = \sum_{j=1}^m \hat{u}_j \hat{v}_j^T \hat{\alpha}_j$:
\begin{align}
    |L(\hat{Z}) - L(Z^*)| \leq L_c \left\|  \begin{bmatrix}  x_1^T (\sum_{j=1}^m \hat{u}_j \hat{v}_j^T \frac{\rho^* \pi}{\gamma m} - 2Z^*) x_1 \\ \vdots \\  x_n^T (\sum_{j=1}^m \hat{u}_j \hat{v}_j^T \frac{\rho^* \pi}{\gamma m} - 2Z^*) x_n \end{bmatrix} \right\|_{\infty} 
\end{align}
%
where we substituted $\hat{\alpha}_j = \rho^* \frac{\pi}{\gamma m}$. The scaling factor of $2$ in front of $Z^*$ is due to the scaling factor in the SDP, i.e., $\hat{y}_i=2x_i^TZx_i$. 
Plugging in $Z^* / \rho^* = Z_s^*=\frac{\pi}{2\gamma}\Exs[u_1v_1^T]$ yields 
\begin{align}
    |L(\hat{Z}) - L(Z^*)| &\leq L_c \left\| \frac{\rho^* \pi}{\gamma}  \begin{bmatrix} x_1^T (\frac{1}{m}\sum_{j=1}^m \hat{u}_j \hat{v}_j^T - \Exs[u_1v_1^T]) x_1 \\ \vdots \\ x_n^T (\frac{1}{m}\sum_{j=1}^m \hat{u}_j \hat{v}_j^T  - \Exs[u_1v_1^T]) x_n \end{bmatrix} \right\|_{\infty} \nonumber \\
    &= L_c \frac{\rho^* \pi}{\gamma} \max_{i=1,\dots,n} \big| x_i^T (\frac{1}{m}\sum_{j=1}^m \hat{u}_j \hat{v}_j^T - \Exs[u_1v_1^T]) x_i \big| \nonumber \\
    &\leq L_c \frac{\rho^* \pi}{\gamma} \max_{i=1,\dots,n} (\epsilon \|x_i\|_2^2) =  L_c \frac{\rho^* \pi}{\gamma} \epsilon R_m^2
\end{align}
which holds with probability at least $1-\exp\left( - \frac{m\epsilon^2}{(2\gamma / \pi)^2 \|Z_s^*\|_2^2 + d(c^\prime + 2\epsilon/3)}+\log(2d)\right)$ as a result of the matrix Bernstein concentration bound.
Therefore, when the number of sampled neurons satisfies the inequality
\begin{align*}
     \frac{m\epsilon^2}{(2\gamma / \pi)^2 \|Z^*\|_2^2 + d(c^\prime + 2\epsilon/3)}\ge 2\log(2d) \,,
\end{align*}
this probability is at least $1-\exp(-\log(2d))=1-\exp(-C \epsilon^2 m/d)$, where $C>0$ is a constant independent of $d$, $m$ and $\epsilon$.

Next, we obtain upper and lower bounds on the non-convex optimal value. Since the SDP solution provides a lower bound, and the sampled quantized network provides an upper bound, we can bound the optimal value of the original non-convex problem as follows
\begin{align} \label{eq:bounds_nonconvex}
    L(\hat{Z}) + \beta R(\hat{Z}) \geq L(Z_{nc}^*) + \beta R(Z_{nc}^*) \geq L(Z^*) + \beta R(Z^*) \,.
\end{align}
We have already established that $|L(\hat{Z}) - L(Z^*)| \leq \frac{\rho^* \pi}{\gamma} L_cR_m^2\epsilon$ with high probability. It follows
\begin{align}
    L(\hat{Z}) - L(Z_{nc}^*) &= L(\hat{Z}) - L(Z^*) + L(Z^*) - L(Z_{nc}^*) \nonumber \\ 
    &\leq \frac{\rho^* \pi}{\gamma}L_cR_m^2\epsilon + L(Z^*) - L(Z_{nc}^*) \nonumber \\
    &\leq \frac{\rho^* \pi}{\gamma}L_cR_m^2\epsilon + \beta R(Z_{nc}^*)
\end{align}
where we have used \eqref{eq:bounds_nonconvex} and that $R(Z^*) \geq 0$ to obtain the last inequality. Furthermore, \eqref{eq:bounds_nonconvex} implies that $L(Z_{nc}^*) - L(\hat{Z}) \leq \beta R(\hat{Z})$. 
If we pick the regularization coefficient $\beta$ such that it satisfies $\beta \leq \frac{\frac{\rho^* \pi}{\gamma} L_cR_m^2\epsilon}{R(Z_{nc}^*)}$ and $\beta \leq \frac{\frac{\rho^* \pi}{\gamma} L_cR_m^2\epsilon}{R(\hat{Z})}$, we obtain the following approximation error bound
\begin{align}
    |L(Z_{nc}^*) - L(\hat{Z})| \leq 2\frac{\rho^* \pi}{\gamma} L_cR_m^2\epsilon \,.
\end{align}
Rescaling $\epsilon$ by $2\frac{\rho^* \pi}{\gamma} L_cR_m^2$, i.e., replacing $\epsilon$ with $\frac{1}{2\frac{\rho^* \pi}{\gamma} L_cR_m^2}\epsilon$, we obtain the claimed approximation result.
\end{proof}

\subsection{Duality Analysis for Bilinear Activation} \label{sec:dual_analysis_bilinear}
This subsection has the details of the duality analysis that we have carried out to obtain the SDP in \eqref{eq:sdpbilinear} for the bilinear activation architecture. The derivations follow the same strategy as the duality analysis in Section \ref{sec:duality}. The non-convex problem for training such a network is stated as follows:
\begin{align}
    p^*_b = &\min_{\mbox{s.t.}\, u_j,v_j\in\{-1,1\}^d,\alpha_j\in\mathbb{R}\,\forall j\in[m]}  \ell \left(\sum_{j=1}^m ((Xu_j) \circ (Xv_j)) \alpha_j, \, y \right) + \beta d \sum_{j=1}^m \vert\alpha_j\vert \,.
\end{align}
Taking the convex dual with respect to the second layer weights $\{\alpha_j\}_{j=1}^m$, the optimal value of the primal is lower bounded by
\begin{align}
    p^* \geq d^* = \max_{\max_{u,v\in\{-1,1\}^d} |\nu^T ((Xu)\circ (Xv))| \leq \beta d} -\ell^*(-\nu)
\end{align}
where $\nu \in \mathbb{R}^n$ is the dual variable. 

The constraint $\max_{u,v\in\{-1,1\}^d} |\nu^T ((Xu)\circ (Xv))| \leq \beta d$ can be equivalently stated as the following two inequalities
\begin{align}
    q_1^* &= \max_{u_i^2=v_i^2=1,\forall i} u^T \left(\sum_{i=1}^n \nu_ix_ix_i^T\right) v \leq \beta d \,,  \nonumber \\
    q_2^* &= \max_{u_i^2=v_i^2=1,\forall i} u^T \left(-\sum_{i=1}^n \nu_ix_ix_i^T\right) v \leq \beta d.
\end{align}
We note that the second constraint $q_2^* \leq \beta d$ is redundant since the change of variable $u \leftarrow -u$ in the first constraint leads to the second constraint:
\begin{align}
   q_1^*= \max_{u_i^2=v_i^2=1,\forall i} u^T \left(\sum_{i=1}^n \nu_ix_ix_i^T\right) v =\max_{(-u_i)^2=v_i^2=1,\forall i} (-u)^T \left(\sum_{i=1}^n \nu_ix_ix_i^T\right) v = \max_{u_i^2=v_i^2=1,\forall i} u^T \left(-\sum_{i=1}^n \nu_ix_ix_i^T\right) v = q_2^*.
\end{align}
In the sequel, we remove the redundant constraint $q_2^* \leq \beta d$.
The SDP relaxation for the maximization $\max_{u_i^2=v_i^2=1,\forall i} u^T \left(\sum_{i=1}^n \nu_ix_ix_i^T\right) v$ is given by (see, e.g., \cite{alon2004approximating})
\begin{align} \label{eq:q_1_hat}
    \hat{q}_1 = \max_{K=\begin{bmatrix} V & Z \\ Z^T & W \end{bmatrix} \succeq 0,\, K_{jj}=1,\forall j} \tr\left( \sum_{i=1}^n \nu_i x_ix_i^T Z \right) \,.
\end{align}
%



\noindent The dual of the above SDP relaxation can be derived via standard convex duality theory, and can be stated as
\begin{align}
    \min_{z_1,z_2\mbox{ s.t. }\ones^Tz_1+\ones^Tz_2=0}  2d\, \lambda_{\max}\left(\begin{bmatrix} \diag(z_1) & \sum_{i=1}^n \nu_i x_ix_i^T  \\ \sum_{i=1}^n \nu_i x_ix_i^T  & \diag(z_2) \end{bmatrix}\right).
\end{align}
%

\noindent Then, we arrive at
\begin{align}
    d^* \geq d_{SDP} := \max_{\nu,z_1,z_2} \quad &-\ell^*(-\nu) \nonumber \\
    \mbox{s.t.} \quad &\begin{bmatrix} \diag(z_1) & \sum_{i=1}^n \nu_i x_ix_i^T  \\ \sum_{i=1}^n \nu_i x_ix_i^T  & \diag(z_2) \end{bmatrix} - \frac{\beta}{2} I \preceq 0 \nonumber \\
    & \ones^Tz_1+\ones^Tz_2=0 \,. 
\end{align}
Next, we will find the dual of the above problem. The Lagrangian is given by
\begin{align}
    &L(\nu, z_1, z_2, Q, \rho) = \nonumber \\
    &= -\ell^*(-\nu) - \tr\left(Q\begin{bmatrix} \diag(z_1) & \sum_{i=1}^n \nu_i x_ix_i^T  \\ \sum_{i=1}^n \nu_i x_ix_i^T  & \diag(z_2) \end{bmatrix} \right) + \frac{\beta}{2} \tr(Q) + \rho \sum_{j=1}^d (z_{1,j} +z_{2,j}) \nonumber \\
    &= -\ell^*(-\nu) - \sum_{j=1}^d \left(V_{jj}z_{1,j} + W_{jj}z_{2,j}\right)- 2\sum_{i=1}^n \nu_ix_i^TZx_i + \frac{\beta}{2} \tr(Q) + \rho \sum_{j=1}^d (z_{1,j} + z_{2,j})
\end{align} 
%
Maximizing the Lagrangian with respect to $\nu, z_1,z_2$ yields the problem
\begin{align}
    \min_{Q, \rho} \quad &\ell\left(\begin{bmatrix} 2x_1^TZx_1 \\ \vdots \\ 2x_n^TZx_n \end{bmatrix},\, y\right) + \frac{\beta}{2} \tr(Q)\nonumber \\
    \mbox{s.t.}\quad & V_{jj} = \rho, \, W_{jj} = \rho, \,j=1,\dots,d
    \nonumber \\ 
    & Q=\begin{bmatrix} V & Z \\ Z^T & W \end{bmatrix} \succeq 0 \,.
\end{align}

\noindent Finally, we obtain the following more concise form for the convex program
\begin{align}
  \min_{Q, \rho} \quad & \ell \left(\hat{y}, \, y \right) + \beta d \rho \nonumber \\
    \mbox{s.t.} \quad & \hat{y}_i = 2x_i^TZx_i, \, i=1,\dots,n \nonumber \\
    & Q_{jj} = \rho, \, j=1,\dots,2d\nonumber \\
    &Q=\begin{bmatrix} V & Z \\ Z^T & W \end{bmatrix} \succeq 0 \,.
\end{align}

\input{vector_output_case.tex}

\section{Further Details on Step 4 of the Sampling Algorithm}

As stated in Step 4 of the sampling algorithm given in subsection \ref{subsec:sampling_alg}, it is possible to transform the bilinear activation architecture to a quadratic activation neural network with $3m$ neurons. The first layer weights of the quadratic activation network can be obtained, via the symmetrization identity, as $1/2(u_j+v_j) \in \{-1,0,+1\}^d$, $u_j \in \{-1,+1\}^d$, $v_j \in \{-1,+1\}^d$, $j=1,\dots,m$. The second layer weights are picked as stated in Step 3 for the first $m$ neurons and the remaining $2m$ neurons have the opposite sign.

\section{Additional Numerical Results}
Figure \ref{fig:realdata_exp_creditapproval} shows the accuracy against time for the credit approval dataset. For this dataset, we similarly observe shorter run times and better classification accuracies for the SDP based sampling method. Furthermore, increasing the number of neurons (plots c,d) improves the accuracy for both methods, which is in consistency with the experiment result shown in Figure \ref{fig:planted_exp_pos_sec_layer}.

\begin{figure} [ht]
\begin{minipage}[b]{0.48\linewidth}
\centering
  \centerline{\includegraphics[width=0.9\columnwidth]{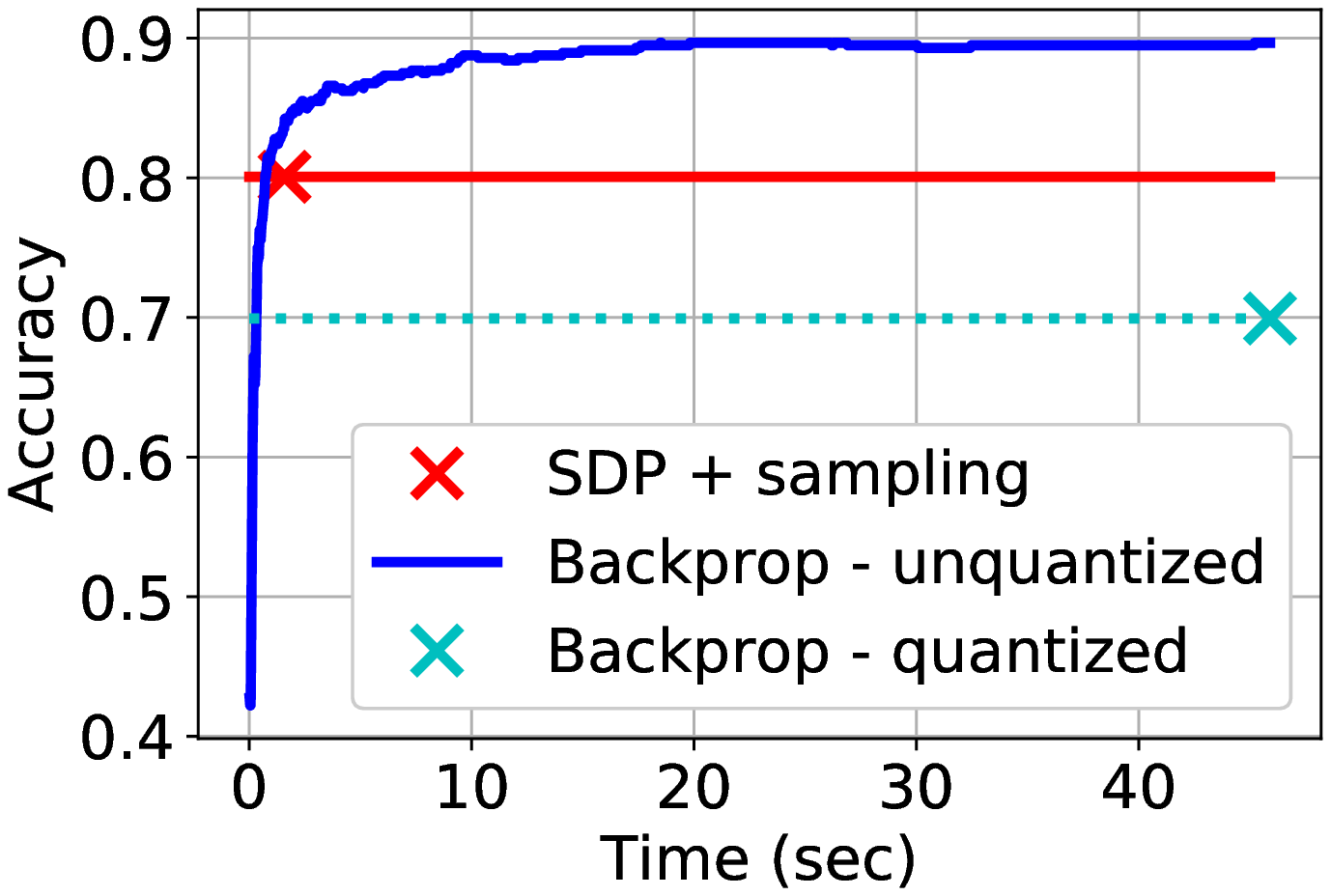}}
  \centerline{a) Training accuracy, $m=500$}
\end{minipage}
\hfill
\begin{minipage}[b]{0.48\linewidth}
\centering
  \centerline{\includegraphics[width=0.9\columnwidth]{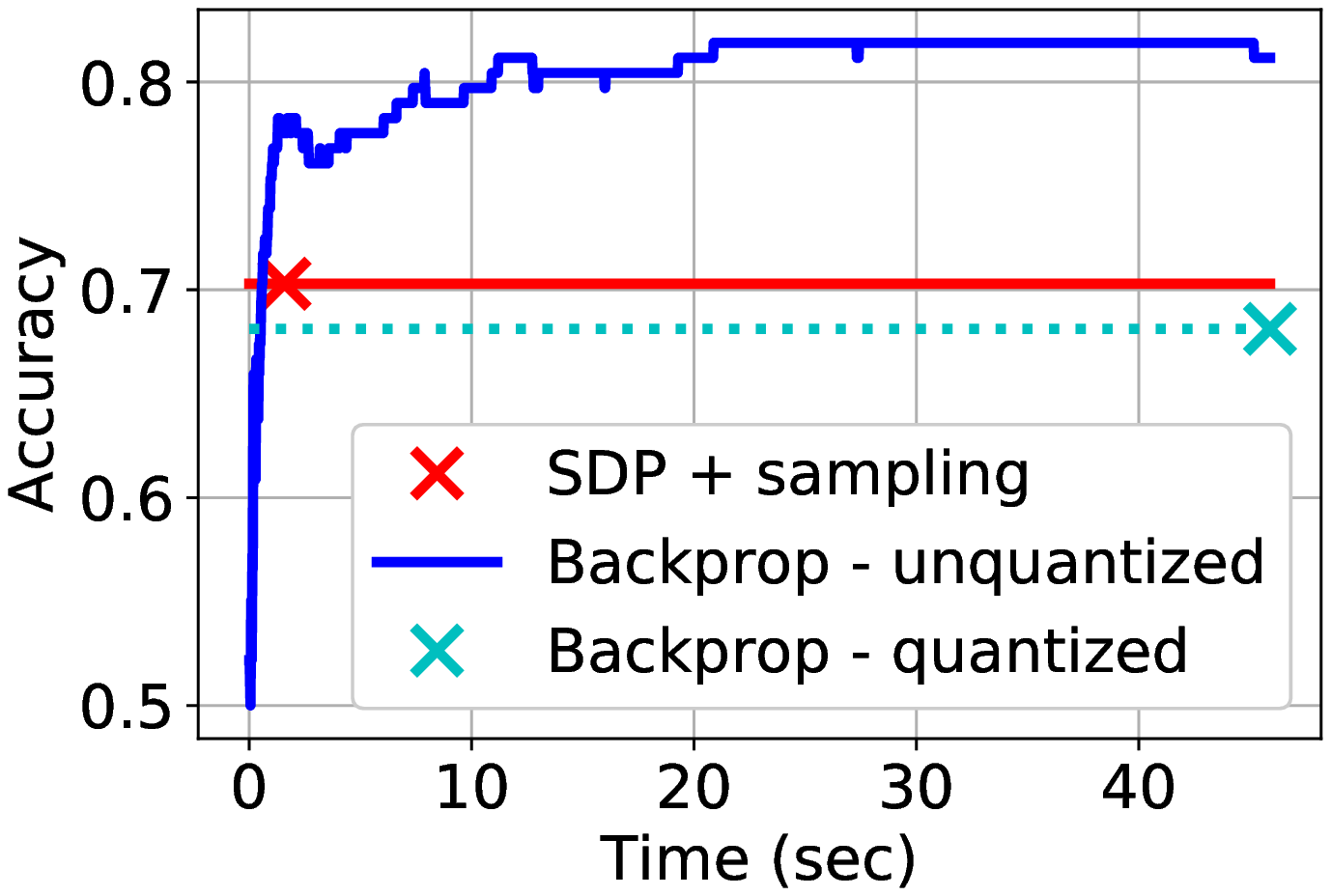}}
  \centerline{b) Test accuracy, $m=500$}
\end{minipage}

\begin{minipage}[b]{0.48\linewidth}
\centering
  \centerline{\includegraphics[width=0.9\columnwidth]{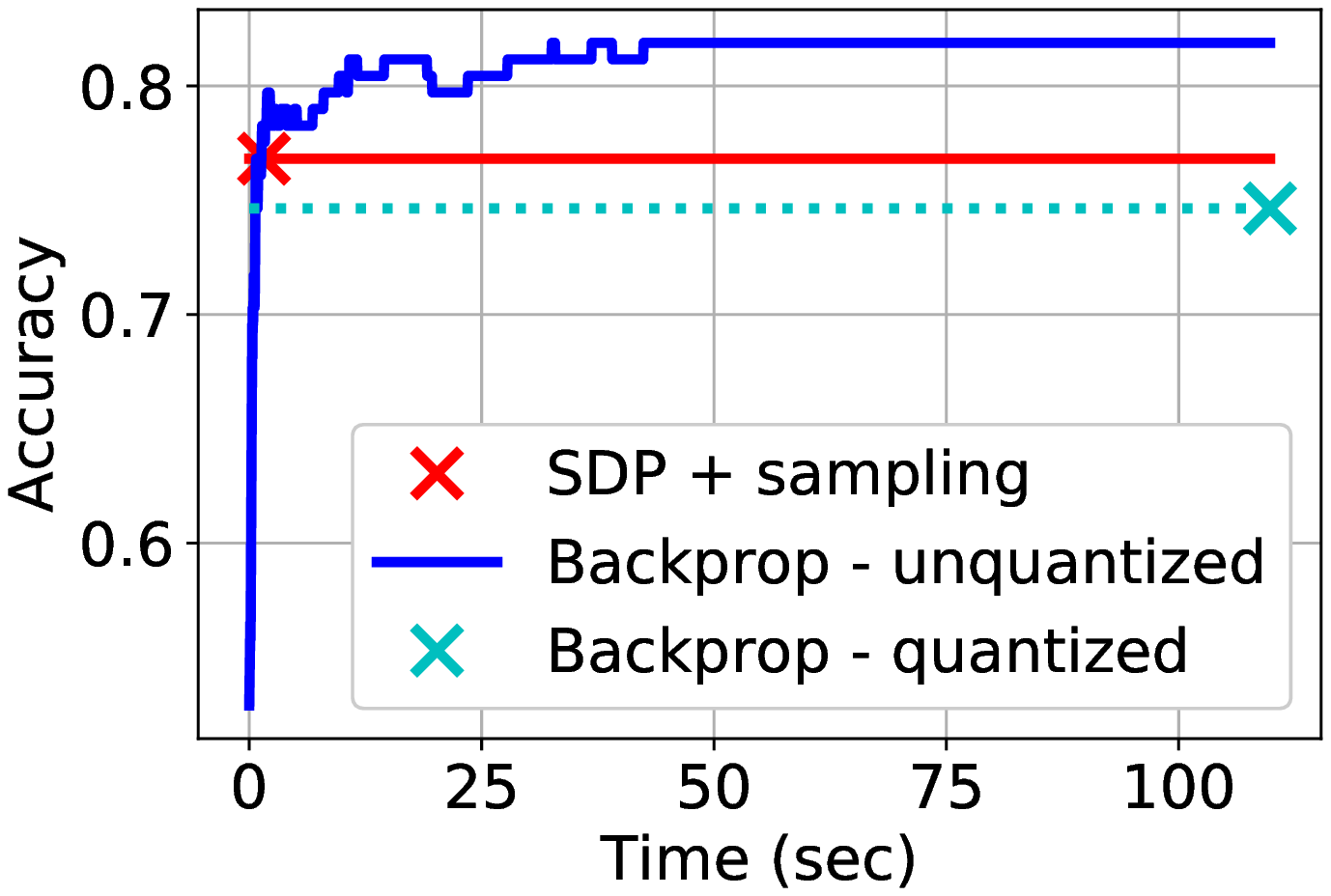}}
  \centerline{c) Training accuracy, $m=2500$}
\end{minipage}
\hfill
\begin{minipage}[b]{0.48\linewidth}
\centering
  \centerline{\includegraphics[width=0.9\columnwidth]{src_v2/v2figures_testacc_creditapproval_squared_n552_d15_beta_cvx-1_bp3_m5000.eps}}
  \centerline{d) Test accuracy, $m=2500$}
\end{minipage}
\caption{Classification accuracy against wall-clock time. Credit approval dataset with $n=552,d=15$. The number of neurons $m$ is specified in the sub-caption for each plot. The regularization coefficient is $\beta=10$ for the SDP based method and $\beta=0.001$ for backpropagation.}
\label{fig:realdata_exp_creditapproval}
\end{figure}

\subsection{ReLU network comparison}
Figure \ref{fig:relu_experiment} compares the SDP based sampling method with a two-layer ReLU network. We train the ReLU network using backpropagation and quantize the first layer weights post-training. The second layer weights are only scaled to account for the quantization of the first layer weights and not restricted to be identical. Thus, unlike the previous figures, the comparison in Figure \ref{fig:relu_experiment} unfairly favors the ReLU network. We observe that the SDP approach can still outperform SGD in this case.

\begin{figure} [ht]
\begin{minipage}[b]{0.48\linewidth}
\centering
  \centerline{\includegraphics[width=0.9\columnwidth]{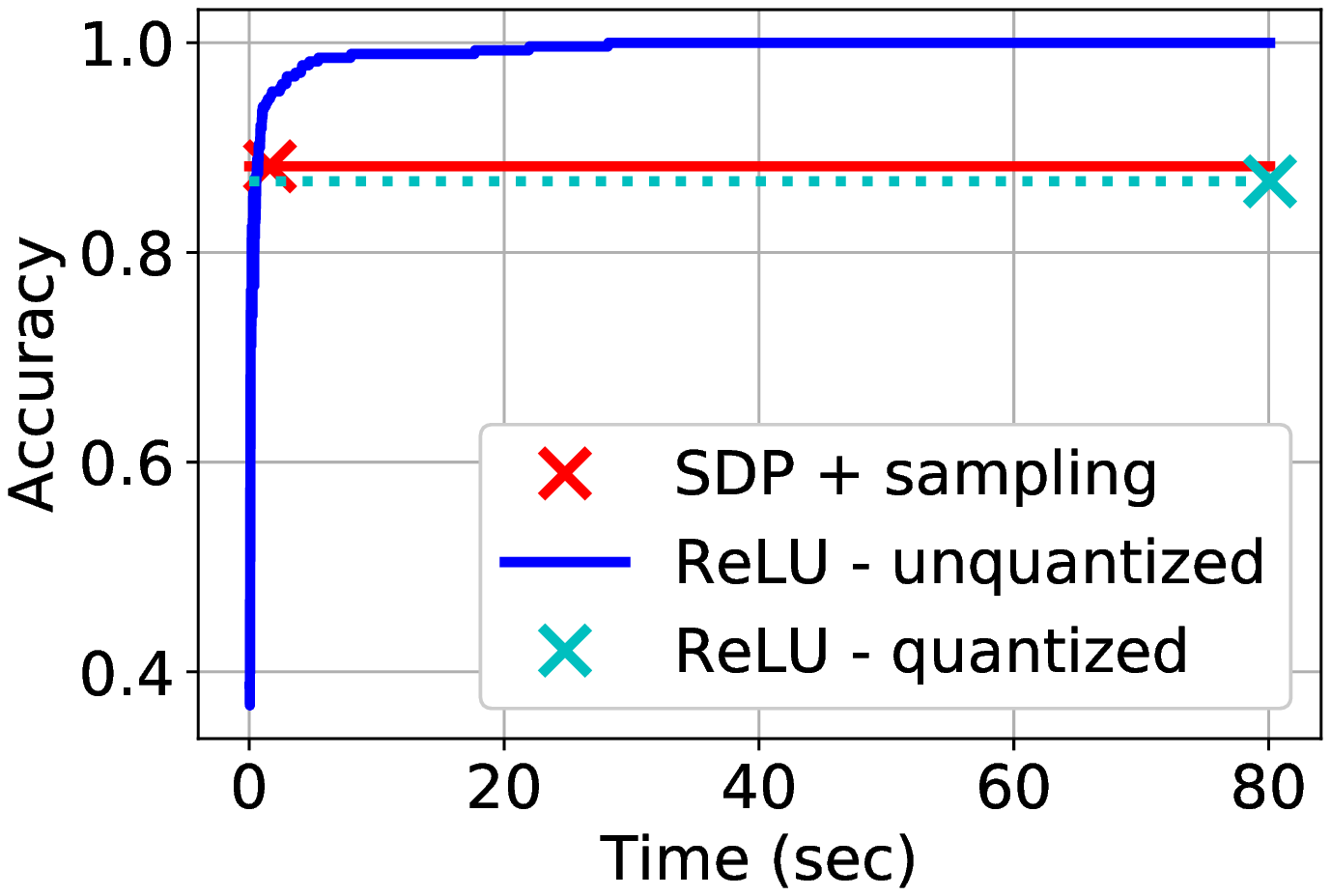}}
  \centerline{a) Training accuracy}
\end{minipage}
\hfill
\begin{minipage}[b]{0.48\linewidth}
\centering
  \centerline{\includegraphics[width=0.9\columnwidth]{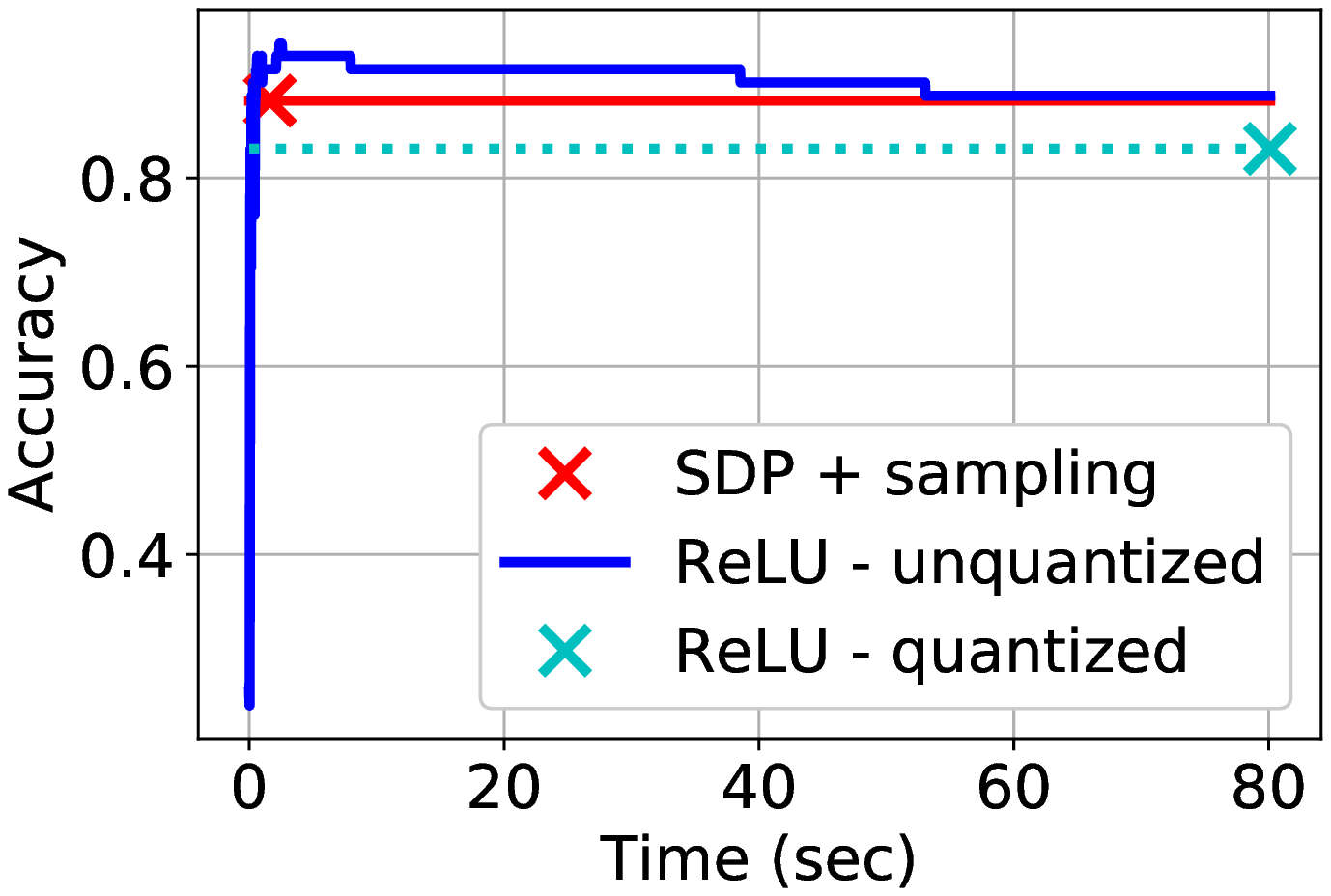}}
  \centerline{b) Test accuracy}
\end{minipage}
\caption{Classification accuracy against wall-clock time showing comparison with a two-layer ReLU network. Ionosphere dataset with $n=280,d=33$. For the SDP based sampling method, $m=2500$ and the regularization coefficient is $\beta=10$. For the ReLU network, $m=5000$ and $\beta=10^{-7}$.}
\label{fig:relu_experiment}
\end{figure}

%% file: vector_output_case.tex
\section{Vector Output Networks} \label{sec:vector_output} 

We will assume the following vector output neural network architecture with bilinear activation
\begin{align}
    f(x) = \sum_{j=1}^m (x^T u_j)(x^T v_j) \alpha_j^T 
\end{align}
where the second layer weights $\alpha_j \in \mathbb{R}^C$, $j=1,\dots,m$ are $C$-dimensional vectors. We note that $f(x):\mathbb{R}^d \rightarrow \mathbb{R}^{1\times C}$. The output of the neural network for all the samples in the dataset can be concisely represented as $\hat{Y} = f(X) \in \mathbb{R}^{n\times C}$. We use $Y \in \mathbb{R}^{n \times C}$ to denote the target matrix. The training problem can be formulated as
\begin{align}
    p^* = \min_{u_j,v_j\in\{-1,1\}^d,\alpha_j\in\mathbb{R}^C\,j\in[m]} \ell \left(\sum_{j=1}^m ((Xu_j) \circ (Xv_j)) \alpha_j^T, \, Y \right) + \beta d \sum_{j=1}^m \|\alpha_j\|_1 \,.
\end{align}
Or,
\begin{align}
    p^* = \min_{u_j,v_j\in\{-1,1\}^d, \,j\in[m]} \min_{\alpha_j \in \mathbb{R}^C , \,j\in[m], \, \hat{Y}} \ell\left(\hat{Y}, Y\right) + \beta d \sum_{j=1}^m \|\alpha_j\|_1 \quad \mbox{s.t.} \quad \hat{Y} = \sum_{j=1}^m ((Xu_j) \circ (Xv_j)) \alpha_j^T \,.
\end{align}
The dual problem for the inner minimization problem is
\begin{align}
    \max_{\nu} -\ell^*(-\nu) \quad \mbox{s.t.} \quad |\nu_k^T ((Xu_j) \circ (Xv_j))| \leq \beta d ,\, \forall j,k \,.
\end{align}
We have introduced the dual variable $\nu \in \mathbb{R}^{n \times C}$ and its columns are denoted by $\nu_k \in \mathbb{R}^n$, $k=1,\dots,C$.
The optimal value of the primal is lower bounded by
\begin{align}
    p^* \geq d^* = \max_{\max_{u, v\in \{-1,1\}^d} |\nu_k^T ((Xu) \circ (Xv))| \leq \beta d \,, \forall k} -\ell^*(-\nu) \,.
\end{align}
The constraints of the above optimization problem can be equivalently stated as the following inequalities
\begin{align}
    q_{1,k}^* &= \max_{u_i^2=v_i^2=1,\forall i} u^T \left(\sum_{i=1}^n \nu_{k,i}x_ix_i^T\right) v \leq \beta d, \,\, k=1,\dots,C,  \nonumber \\
    q_{2,k}^* &= \max_{u_i^2=v_i^2=1,\forall i} u^T \left(-\sum_{i=1}^n \nu_{k,i}x_ix_i^T\right) v \leq \beta d, \,\, k=1,\dots,C \,.
\end{align}
As we have shown in Section \ref{sec:dual_analysis_bilinear}, the second set of inequalities $q_{2,k}^* \leq \beta d$ are implied by the first and hence we remove them. The SDP relaxation for the maximization $\max_{u_i^2=v_i^2=1,\forall i} u^T \left(\sum_{i=1}^n \nu_{k,i}x_ix_i^T\right) v$ is given by
\begin{align}
    \hat{q}_{1,k} = \max_{K=\begin{bmatrix} V & Z \\ Z^T & W \end{bmatrix} \succeq 0,\, K_{jj}=1,\forall j} \tr\left( \sum_{i=1}^n \nu_{k,i} x_ix_i^T Z \right) \,.
\end{align}
We have previously given the dual of this problem as
\begin{align}
    \min_{z_{k,1},z_{k,2}\mbox{ s.t. }\ones^Tz_{k,1}+\ones^Tz_{k,2}=0}  2d\, \lambda_{\max}\left(\begin{bmatrix} \diag(z_{k,1}) & \sum_{i=1}^n \nu_{k,i} x_ix_i^T  \\ \sum_{i=1}^n \nu_{k,i} x_ix_i^T  & \diag(z_{k,2}) \end{bmatrix}\right) \,,
\end{align}
where we define the variables $z_{k,1}\in\mathbb{R}^d, z_{k,2}\in\mathbb{R}^d$, $k=1,\dots,C$. This allows us to establish the following lower bound
\begin{align}
    d^* \geq d_{SDP} := \max_{\nu,\{z_{k,1},z_{k,2}\}_{k=1}^C} \quad &-\ell^*(-\nu) \nonumber \\
    \mbox{s.t.} \quad &\begin{bmatrix} \diag(z_{k,1}) & \sum_{i=1}^n \nu_{k,i} x_ix_i^T  \\ \sum_{i=1}^n \nu_{k,i} x_ix_i^T  & \diag(z_{k,2}) \end{bmatrix} - \frac{\beta}{2} I \preceq 0, \quad k=1,\dots,C \nonumber \\
    & \ones^Tz_{k,1}+\ones^Tz_{k,2}=0, \quad k=1,\dots,C \,.
\end{align}
Next, we find the dual of this problem. First, we write the Lagrangian:
\begin{align}
    &L(\nu, \{z_{k,1}, z_{k,2}, Q_k, \rho_k\}_{k=1}^C) = \nonumber \\
    &= -\ell^*(-\nu) - \sum_{k=1}^C \tr\left(Q_k\begin{bmatrix} \diag(z_{k,1}) & \sum_{i=1}^n \nu_{k,i} x_ix_i^T  \\ \sum_{i=1}^n \nu_{k,i} x_ix_i^T  & \diag(z_{k,2}) \end{bmatrix} \right) + \frac{\beta}{2} \sum_{k=1}^C \tr(Q_k) + \sum_{k=1}^C \rho_k (\ones^Tz_{k,1}+\ones^Tz_{k,2}) \nonumber \\
    &= -\ell^*(-\nu) - \sum_{k=1}^C \left( \diag(V_k)^T z_{k,1} + \diag(W_k)^T z_{k,2} \right) - 2\sum_{k=1}^C \sum_{i=1}^n \nu_{k,i} x_i^TZ_kx_i + \frac{\beta}{2} \sum_{k=1}^C \tr(Q_k) \nonumber \\
    &\quad + \sum_{k=1}^C \rho_k (\ones^Tz_{k,1}+\ones^Tz_{k,2}) \,,
\end{align}
where we have introduced $Q_k=\begin{bmatrix} V_k & Z_k \\ Z_k^T & W_k \end{bmatrix}$.
Maximization of the Lagrangian with respect to $\nu, z_{k,1}, z_{k,2}$, $k=1,\dots,C$ leads to the dual problem given by
\begin{align}
    \min_{\{Q_k, \rho_k\}_{k=1}^C} \quad &\ell\left(\begin{bmatrix} 2x_1^TZ_1x_1 & \dots & 2x_1^TZ_Cx_1 \\ \vdots & \ddots & \vdots \\ 2x_n^TZ_1x_n & \dots & 2x_n^TZ_Cx_n \end{bmatrix},\, Y\right) + \frac{\beta}{2} \sum_{k=1}^C \tr(Q_k) \nonumber \\
    \mbox{s.t.}\quad & V_{k,jj} = \rho_k, \, W_{k,jj} = \rho_k, \quad k \in [C], \,\,  j \in [d]
    \nonumber \\ 
    & Q_k=\begin{bmatrix} V_k & Z_k \\ Z_k^T & W_k \end{bmatrix} \succeq 0, \quad k \in [C].
\end{align}
More concisely,
\begin{align} \label{eq:vector_output_sdp}
    \min_{\{Q_k, \rho_k\}_{k=1}^C} \quad &\ell\left(\hat{Y},\, Y\right) + \beta d \sum_{k=1}^C \rho_k \nonumber \\
    \mbox{s.t.}\quad & \hat{Y}_{ik} = 2x_i^TZ_kx_i,  \quad i \in [n], \,\, k \in [C] \nonumber \\
    & Q_{k,jj} = \rho_k, \quad k \in [C], \,\, j \in [2d]
    \nonumber \\ 
    & Q_k=\begin{bmatrix} V_k & Z_k \\ Z_k^T & W_k \end{bmatrix} \succeq 0, \quad k \in [C].
\end{align}
where $V_k,Z_k,W_k$ are $d\times d$-dimensional matrices.

\subsection{Sampling Algorithm for Vector Output Networks}
We now give the sampling algorithm:

\begin{enumerate}
    \item Solve the SDP in \eqref{eq:vector_output_sdp} and define the matrices $Z_{s,k}^* \leftarrow Z_k^* / \rho_k^*$, $k=1,\dots,C$.
    
    \item Find $Q_k^*$, $k=1,\dots,C$ by solving the problem 
    \begin{align} \label{eq:cov_matrix_Q_problem_vectorout}
        Q_k^* := \arg \min_{Q\succeq 0, Q_{jj}=1\forall j} \|Q_{(12)} -  \sin(\gamma Z_{s,k}^*) \|_F^2 \,.
    \end{align}
    
    \item Carry out the following steps for each $k=1,\dots,C$:
    \begin{enumerate}[label=\alph*.]
        \item Sample $m/C$ pairs of the first layer weights $u_j,v_j$ via $\begin{bmatrix} u_j \\ v_j \end{bmatrix} \sim \sign( \mathcal{N}(0,Q_k^*))$.
        
        \item Set the second layer weights for these neurons to $\alpha_j = \rho_k^* C \frac{\pi}{\gamma m}e_k$ where $e_k \in \mathbb{R}^C$ is the $k$'th unit vector.
    \end{enumerate}
    
    \item (optional) Transform the quantized bilinear activation network to a quantized polynomial activation network as described in Section \ref{sec:lifting}.
\end{enumerate}

Figure \ref{fig:vector_out} shows the classification accuracy on a UCI machine learning repository with $C=4$ classes. We perform one-hot encoding on the output and use the vector output SDP and sampling method developed in this section. We observe that the accuracy of the sampling method approaches the accuracy of the lower bounding SDP as $m$ is increased.

\begin{figure} [ht]
\begin{minipage}[b]{0.48\linewidth}
\centering
  \centerline{\includegraphics[width=0.9\columnwidth]{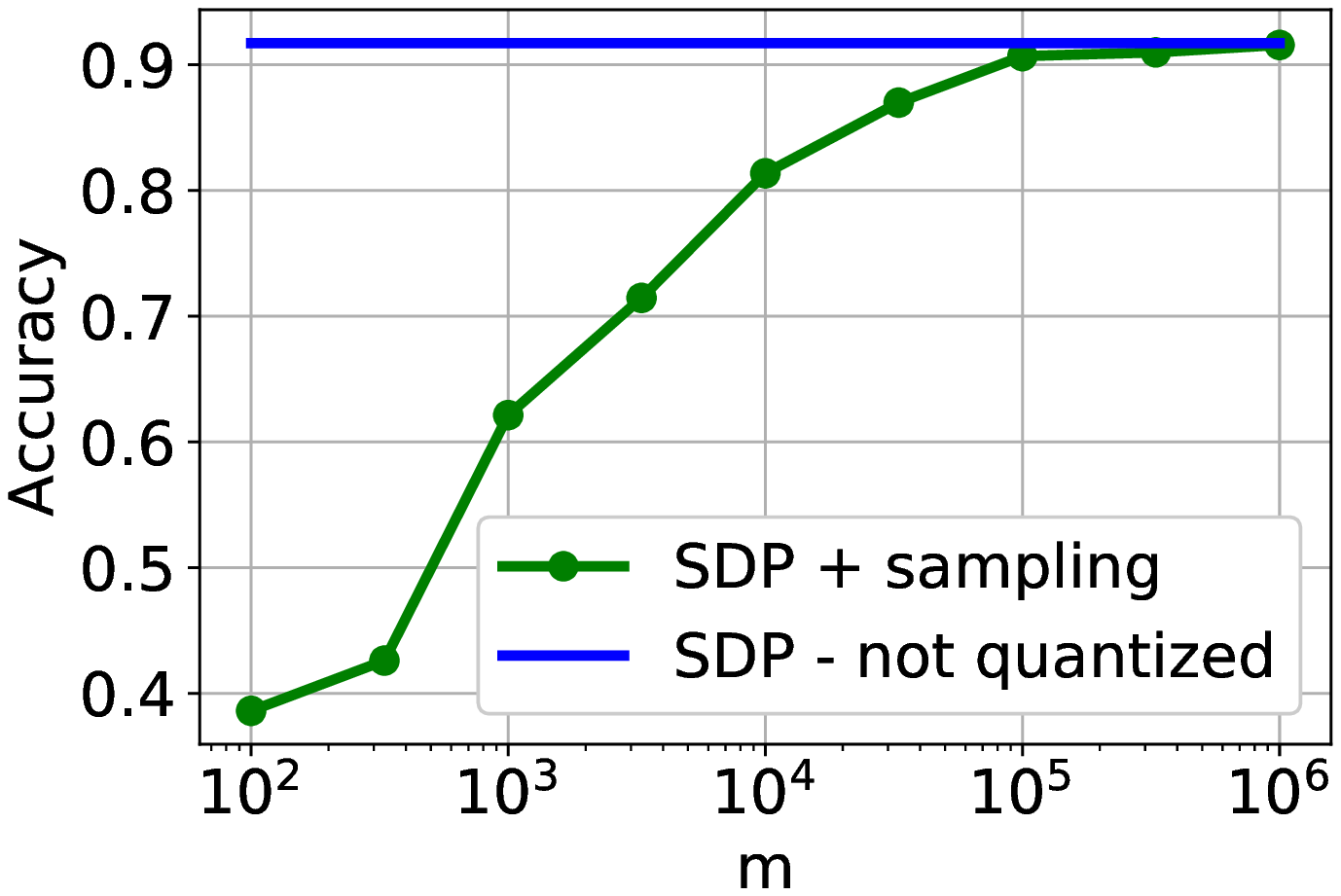}}
  \centerline{a) Training accuracy}
\end{minipage}
\hfill
\begin{minipage}[b]{0.48\linewidth}
\centering
  \centerline{\includegraphics[width=0.9\columnwidth]{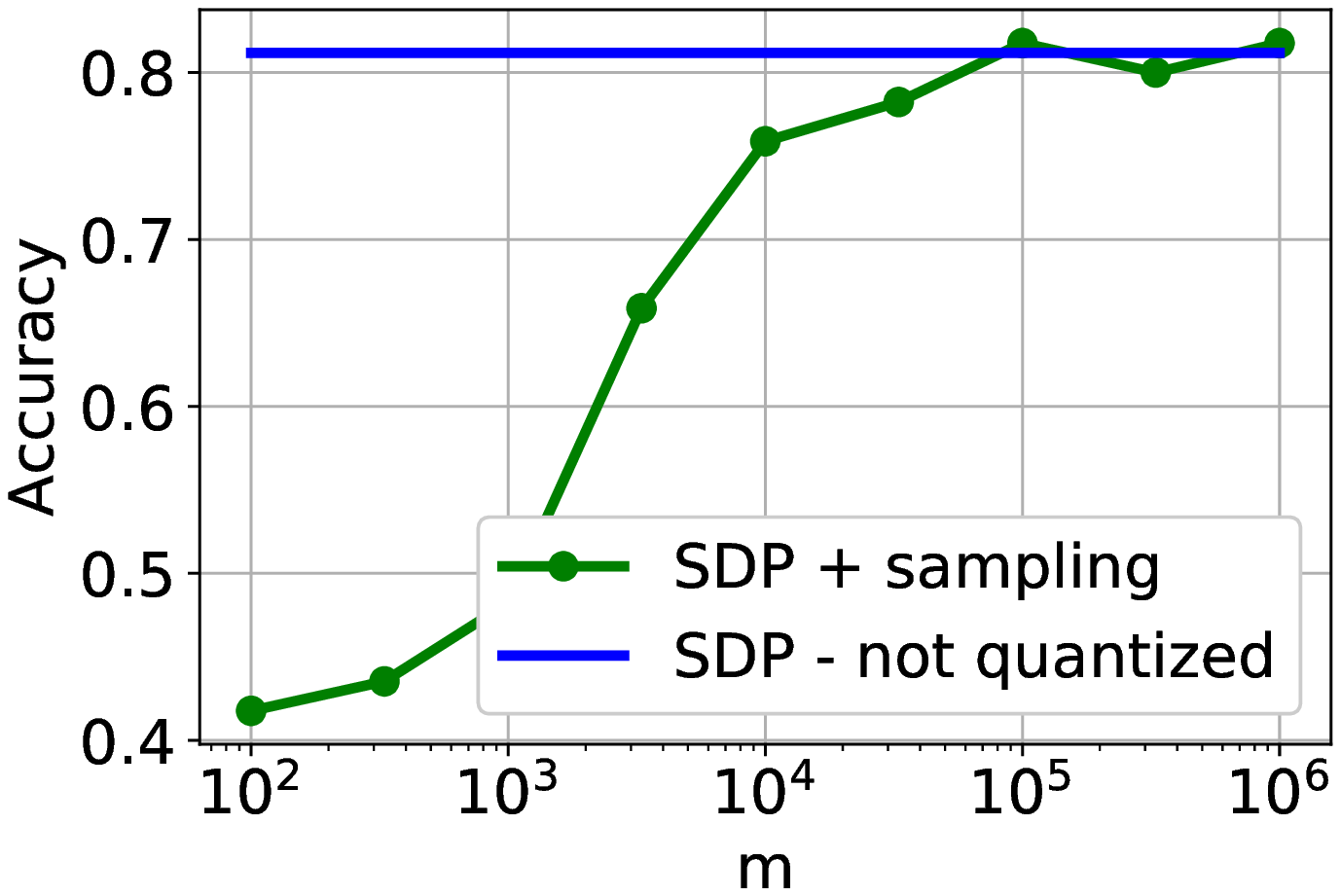}}
  \centerline{b) Test accuracy}
\end{minipage}
\caption{Vector output network experiment showing multiclass classification accuracy against the number of sampled neurons $m$. The dataset is statlog vehicle multiclass with $C=4$ classes and dimensions $n=676,d=18$. The regularization coefficient is $\beta=1$. The blue solid line shows the accuracy when we predict the labels using the lower bounding SDP in \eqref{eq:vector_output_sdp} without quantization. The green curve with circle markers shows the accuracy for the quantized network when we use the sampling method that we designed for the vector output case.}
\label{fig:vector_out}
\end{figure}

%% file: main.bbl
\begin{thebibliography}{10}

\bibitem{agrawal2018rewriting}
Akshay Agrawal, Robin Verschueren, Steven Diamond, and Stephen Boyd.
\newblock A rewriting system for convex optimization problems.
\newblock {\em Journal of Control and Decision}, 5(1):42--60, 2018.

\bibitem{allenzhu2020backward}
Zeyuan Allen-Zhu and Yuanzhi Li.
\newblock Backward feature correction: How deep learning performs deep
  learning.
\newblock {\em arXiv preprint arXiv:2001.04413}, 2020.

\bibitem{alon2004approximating}
Noga Alon and Assaf Naor.
\newblock Approximating the cut-norm via grothendieck's inequality.
\newblock In {\em Proceedings of the thirty-sixth annual ACM symposium on
  Theory of computing}, pages 72--80, 2004.

\bibitem{anwar2015}
S.~{Anwar}, K.~{Hwang}, and W.~{Sung}.
\newblock Fixed point optimization of deep convolutional neural networks for
  object recognition.
\newblock In {\em 2015 IEEE International Conference on Acoustics, Speech and
  Signal Processing (ICASSP)}, pages 1131--1135, April 2015.

\bibitem{bach2017breaking}
Francis Bach.
\newblock Breaking the curse of dimensionality with convex neural networks.
\newblock {\em The Journal of Machine Learning Research}, 18(1):629--681, 2017.

\bibitem{bartan2019convex}
Burak Bartan and Mert Pilanci.
\newblock Convex relaxations of convolutional neural nets.
\newblock In {\em ICASSP 2019-2019 IEEE International Conference on Acoustics,
  Speech and Signal Processing (ICASSP)}, pages 4928--4932. IEEE, 2019.

\bibitem{bartan2021neural}
Burak Bartan and Mert Pilanci.
\newblock Neural spectrahedra and semidefinite lifts: Global convex
  optimization of polynomial activation neural networks in fully
  polynomial-time.
\newblock {\em arXiv preprint arXiv:2101.02429}, 2021.

\bibitem{bengio2006convex}
Yoshua Bengio, Nicolas Le~Roux, Pascal Vincent, Olivier Delalleau, and Patrice
  Marcotte.
\newblock Convex neural networks.
\newblock {\em Advances in neural information processing systems}, 18:123,
  2006.

\bibitem{diamond2016cvxpy}
Steven Diamond and Stephen Boyd.
\newblock {CVXPY}: {A} {P}ython-embedded modeling language for convex
  optimization.
\newblock {\em Journal of Machine Learning Research}, 17(83):1--5, 2016.

\bibitem{uci2019datasets}
Dheeru Dua and Casey Graff.
\newblock {UCI} machine learning repository, 2017.

\bibitem{ergen2019shallow}
T.~{Ergen} and M.~{Pilanci}.
\newblock Convex optimization for shallow neural networks.
\newblock In {\em 2019 57th Annual Allerton Conference on Communication,
  Control, and Computing (Allerton)}, pages 79--83, 2019.

\bibitem{ergen2020aistats}
Tolga Ergen and Mert Pilanci.
\newblock Convex geometry of two-layer relu networks: Implicit autoencoding and
  interpretable models.
\newblock In {\em International Conference on Artificial Intelligence and
  Statistics}, pages 4024--4033. PMLR, 2020.

\bibitem{ergen2020convexdeep}
Tolga Ergen and Mert Pilanci.
\newblock Revealing the structure of deep neural networks via convex duality.
\newblock {\em arXiv preprint arXiv:2002.09773}, 2020.

\bibitem{ergen2020implicit}
Tolga Ergen and Mert Pilanci.
\newblock Implicit convex regularizers of cnn archi-tectures: Convex
  optimization of two-and three-layer networks in polynomial time.
\newblock {\em International Conference on Learning Representations (ICLR),
  arXiv preprint arXiv:2006.14798}, 2021.

\bibitem{ergen2021demystifying}
Tolga Ergen, Arda Sahiner, Batu Ozturkler, John Pauly, Morteza Mardani, and
  Mert Pilanci.
\newblock Demystifying batch normalization in relu networks: Equivalent convex
  optimization models and implicit regularization.
\newblock {\em arXiv preprint arXiv:2103.01499}, 2021.

\bibitem{gilad2016cryptonets}
Ran Gilad-Bachrach, Nathan Dowlin, Kim Laine, Kristin Lauter, Michael Naehrig,
  and John Wernsing.
\newblock Cryptonets: Applying neural networks to encrypted data with high
  throughput and accuracy.
\newblock In {\em International Conference on Machine Learning}, pages
  201--210, 2016.

\bibitem{gwmaxcut}
M.~X. Goemans and D.~P. Williamson.
\newblock Improved approximation algorithms for maximum cut and satisfiability
  problems using semidefinite programming.
\newblock {\em J. Assoc. Comput. Mach.}, 42:1115--1145, 1995.

\bibitem{compress_quantization}
Yunchao Gong, Liu Liu, Ming Yang, and Lubomir~D. Bourdev.
\newblock Compressing deep convolutional networks using vector quantization.
\newblock {\em CoRR}, abs/1412.6115, 2014.

\bibitem{GuptaAGN15}
Suyog Gupta, Ankur Agrawal, Kailash Gopalakrishnan, and Pritish Narayanan.
\newblock Deep learning with limited numerical precision.
\newblock {\em CoRR}, abs/1502.02551, 2015.

\bibitem{guptaexact}
Vikul Gupta, Burak Bartan, Tolga Ergen, and Mert Pilanci.
\newblock Exact and relaxed convex formulations for shallow neural
  autoregressive models.
\newblock {\em International Conference on Acoustics, Speech, and Signal
  Processing}, 2021.

\bibitem{deep_compression}
S.~Han, H.~Mao, and W.J. Dally.
\newblock Deep compression: Compressing deep neural network with pruning,
  trained quantization and huffman coding.
\newblock {\em CoRR}, abs/1510.00149, 2015.

\bibitem{fixedpoint_hwang}
K.~{Hwang} and W.~{Sung}.
\newblock Fixed-point feedforward deep neural network design using weights +1,
  0, and -1.
\newblock In {\em 2014 IEEE Workshop on Signal Processing Systems (SiPS)},
  pages 1--6, Oct 2014.

\bibitem{lacotte2020limiting}
Jonathan Lacotte, Sifan Liu, Edgar Dobriban, and Mert Pilanci.
\newblock Limiting spectrum of randomized hadamard transform and optimal
  iterative sketching methods.
\newblock {\em arXiv preprint arXiv:2002.00864}, 2020.

\bibitem{lacotte2020adaptive}
Jonathan Lacotte and Mert Pilanci.
\newblock Adaptive and oblivious randomized subspace methods for
  high-dimensional optimization: Sharp analysis and lower bounds.
\newblock {\em arXiv preprint arXiv:2012.07054}, 2020.

\bibitem{lacotte2020all}
Jonathan Lacotte and Mert Pilanci.
\newblock All local minima are global for two-layer relu neural networks: The
  hidden convex optimization landscape.
\newblock {\em arXiv preprint arXiv:2006.05900}, 2020.

\bibitem{lacotte2020effective}
Jonathan Lacotte and Mert Pilanci.
\newblock Effective dimension adaptive sketching methods for faster regularized
  least-squares optimization.
\newblock {\em arXiv preprint arXiv:2006.05874}, 2020.

\bibitem{lacotte2020optimal}
Jonathan Lacotte and Mert Pilanci.
\newblock Optimal randomized first-order methods for least-squares problems.
\newblock In {\em International Conference on Machine Learning}, pages
  5587--5597. PMLR, 2020.

\bibitem{lacotte2021fast}
Jonathan Lacotte and Mert Pilanci.
\newblock Fast convex quadratic optimization solvers with adaptive
  sketching-based preconditioners.
\newblock {\em arXiv preprint arXiv:2104.14101}, 2021.

\bibitem{challenges_fixed_point}
Darryl~Dexu Lin and Sachin~S. Talathi.
\newblock Overcoming challenges in fixed point training of deep convolutional
  networks.
\newblock {\em CoRR}, abs/1607.02241, 2016.

\bibitem{fixed_point_quantization}
Darryl~Dexu Lin, Sachin~S. Talathi, and V.~Sreekanth Annapureddy.
\newblock Fixed point quantization of deep convolutional networks.
\newblock {\em CoRR}, abs/1511.06393, 2015.

\bibitem{scs2016paper}
B.~O'Donoghue, E.~Chu, N.~Parikh, and S.~Boyd.
\newblock Conic optimization via operator splitting and homogeneous self-dual
  embedding.
\newblock {\em Journal of Optimization Theory and Applications},
  169(3):1042--1068, June 2016.

\bibitem{scs2016code}
B.~O'Donoghue, E.~Chu, N.~Parikh, and S.~Boyd.
\newblock {SCS}: Splitting conic solver, version 2.1.2.
\newblock \url{https://github.com/cvxgrp/scs}, November 2019.

\bibitem{ozaslan2019iterative}
Ibrahim~Kurban Ozaslan, Mert Pilanci, and Orhan Arikan.
\newblock Iterative hessian sketch with momentum.
\newblock In {\em ICASSP 2019-2019 IEEE International Conference on Acoustics,
  Speech and Signal Processing (ICASSP)}, pages 7470--7474. IEEE, 2019.

\bibitem{pytorch}
Adam Paszke, Sam Gross, Francisco Massa, Adam Lerer, James Bradbury, Gregory
  Chanan, Trevor Killeen, Zeming Lin, Natalia Gimelshein, Luca Antiga, Alban
  Desmaison, Andreas Kopf, Edward Yang, Zachary DeVito, Martin Raison, Alykhan
  Tejani, Sasank Chilamkurthy, Benoit Steiner, Lu~Fang, Junjie Bai, and Soumith
  Chintala.
\newblock Pytorch: An imperative style, high-performance deep learning library.
\newblock {\em Advances in Neural Information Processing Systems 32}, pages
  8024--8035, 2019.

\bibitem{pilanci2020neural}
Mert Pilanci and Tolga Ergen.
\newblock Neural networks are convex regularizers: Exact polynomial-time convex
  optimization formulations for two-layer networks.
\newblock In {\em International Conference on Machine Learning}, pages
  7695--7705. PMLR, 2020.

\bibitem{sahiner2020vector}
Arda Sahiner, Tolga Ergen, John Pauly, and Mert Pilanci.
\newblock Vector-output relu neural network problems are copositive programs:
  Convex analysis of two layer networks and polynomial-time algorithms.
\newblock {\em International Conference on Learning Representations (ICLR),
  arXiv preprint arXiv:2012.13329}, 2021.

\bibitem{sahiner2020convex}
Arda Sahiner, Morteza Mardani, Batu Ozturkler, Mert Pilanci, and John Pauly.
\newblock Convex regularization behind neural reconstruction.
\newblock {\em International Conference on Learning Representations (ICLR),
  arXiv preprint arXiv:2012.05169}, 2021.

\bibitem{tropp2015introduction}
Joel~A Tropp.
\newblock An introduction to matrix concentration inequalities.
\newblock {\em arXiv preprint arXiv:1501.01571}, 2015.

\bibitem{yurtsever2021scalable}
Alp Yurtsever, Joel~A Tropp, Olivier Fercoq, Madeleine Udell, and Volkan
  Cevher.
\newblock Scalable semidefinite programming.
\newblock {\em SIAM Journal on Mathematics of Data Science}, 3(1):171--200,
  2021.

\bibitem{yurtsever2017sketchy}
Alp Yurtsever, Madeleine Udell, Joel Tropp, and Volkan Cevher.
\newblock Sketchy decisions: Convex low-rank matrix optimization with optimal
  storage.
\newblock In {\em Artificial intelligence and statistics}, pages 1188--1196.
  PMLR, 2017.

\bibitem{ternary_quantization}
Chenzhuo Zhu, Song Han, Huizi Mao, and William~J. Dally.
\newblock Trained ternary quantization.
\newblock {\em CoRR}, abs/1612.01064, 2016.

\end{thebibliography}
